\documentclass{article}
\usepackage{ijcai22}
\usepackage[]{ijcai22_extensions}
\usepackage{nccmath}

\newif\ifseparateappendix

\separateappendixtrue

\ifseparateappendix
\usepackage{xr-latexmk}
\myexternaldocument{ijcai22-appendix}
\else
\usepackage[page]{appendix}
\fi

\usepackage{times}

\usepackage{soul}
\usepackage{url}
\usepackage{amsmath}
\usepackage[hidelinks]{hyperref}
\usepackage[utf8]{inputenc}
\usepackage[small]{caption}
\usepackage{graphicx}
\usepackage{booktabs}
\usepackage{gltl_macros}
\usepackage{paper_macros}
\usepackage[]{marginalia}
\usepackage[noframe]{showframe}

\usetikzlibrary{external}

\title{On the (In)Tractability of Reinforcement Learning for LTL Objectives}

\author{
Cambridge Yang$^1$ %
\and
Michael L.\ Littman$^2$\and
Michael Carbin$^1$\\
\affiliations
$^1$MIT CSAIL\\
$^2$Brown University\\
\emails
camyang@csail.mit.edu, 
mlittman@cs.brown.edu,
mcarbin@csail.mit.edu
}

\makeatletter\@input{ijcai22-aux.tex}\makeatother
\makeatletter\@input{ijcai22-appendix-aux.tex}\makeatother

\begin{document}

\maketitle

\begin{abstract}
In recent years, researchers have made significant progress in devising reinforcement-learning algorithms for optimizing linear temporal logic (LTL) objectives and LTL-like objectives.
Despite these advancements, there are fundamental limitations to how well this problem can be solved. Previous studies have alluded to this fact but have not examined it in depth.
In this paper, we address the tractability of reinforcement learning for general LTL objectives from a theoretical perspective.
We formalize the problem under the probably approximately correct learning in Markov decision processes (PAC-MDP) framework, a standard framework for measuring sample complexity in reinforcement learning.
In this formalization, we prove that the optimal policy for any LTL formula is PAC-MDP-learnable if and only if the formula is in the most limited class in the LTL hierarchy, consisting of formulas that are decidable within a finite horizon.
Practically, our result implies that it is impossible for a reinforcement-learning algorithm to obtain a PAC-MDP guarantee on the performance of its learned policy after finitely many interactions with an unconstrained environment for LTL objectives that are not decidable within a finite horizon.
\end{abstract}

\newif\ifshorterpaper

\shorterpapertrue

\shorterpaperfalse %

\section{Introduction}
\label{sec:introduction}

In reinforcement learning, we situate an autonomous agent in an unknown environment and specify an {\em objective}.
We want the agent to learn the optimal behavior for achieving the specified objective by interacting with the environment.

\paragraph{Specifying an Objective}

The objective for the agent is a specification over possible trajectories of the overall system---the environment and the agent.
Each trajectory is an infinite sequence of the states of the system, evolving through time.
The objective specifies which trajectories are desirable so that the agent can identify optimal or near-optimal behaviors with respect to the objective.

\paragraph{The Reward Objective}

One form of an objective is a reward function.
A reward function specifies a scalar value, a reward, for each state of the system.
The desired trajectories are those with higher cumulative discounted rewards.
The reward-function objective is well studied~\citep{sutton98}.
It has desirable properties that allow \reinforcementlearning algorithms to provide performance guarantees on learned behavior \citep{pacmodelfreerl}, meaning that algorithms can guarantee learning behaviors that achieve almost optimal cumulative discounted rewards with high probability.
Due to its versatility, researchers have adopted the reward-function objective as the de facto standard of behavior specification in reinforcement learning.

\subsection{The Linear Temporal Logic Objective} 

However, reward engineering, the practice of encoding desirable behaviors into a reward function, is a difficult challenge in applied reinforcement learning \citep{Dewey2014ReinforcementLA,gltl17}.
To reduce the burden of reward engineering, {\em linear temporal logic} (LTL) has attracted researchers' attention as an alternative objective.

LTL is a formal logic used initially to specify behaviors for system verification \citep{ltlprograms}. An LTL formula is built from a set of propositions about the state of the environment, logical connectives, and temporal operators such as $\always$ (always) and $\eventually$ (eventually).
Many reinforcement-learning tasks are naturally expressible with LTL \citep{gltl17}.
For some classic control examples, we can express: 1) Cart-Pole as $\always \mathit{up}$ (i.e., the pole always stays up), 2) Mountain-Car as $\eventually \mathit{goal}$ (i.e., the car eventually reaches the goal), and
3) Pendulum-Swing-Up as $\eventually\always \mathit{up}$ (i.e., the pendulum eventually always stays up).

Researchers have thus used LTL as an alternative objective specification for reinforcement learning ~\citep{ufukpacmdpltl,dorsa,truncatedltl-rl,omegaregularrl19,hasanbeig2019reinforcement,bozkurt2020control}.
Given an LTL objective specified by an LTL formula, each trajectory of the system either satisfies or violates that formula.
The agent should learn the behavior that maximizes the probability of satisfying that formula.
Moreover, research has shown that using LTL objectives supports automated reward shaping~\citep{osbert,rewardmachine1,ltlrewardshaping}.

\subsection{Trouble with Infinite Horizons}

The general class of LTL objectives consists of {\em infinite\Hyphdash horizon objectives}---objectives that require inspecting infinitely many steps of a trajectory to determine if the trajectory satisfies the objective.
For example, consider the objective $\eventually \mathit{goal}$ (eventually reach the goal). 
Given an infinite trajectory, the objective requires inspecting the entire trajectory in the worst case to determine that the trajectory violates the objective.

Despite the above developments on reinforcement learning with LTL objectives, 
the infinite-horizon nature of these objectives presents challenges that have been alluded to---but not formally treated---in prior work. \citet{smcbltl,pacmdpsmcconfidenceinterval,ltlrewardshaping} noted slow learning times for mastering infinite\Hyphdash horizon properties. 
\citet{gltl17} provided a specific environment that illustrates the intractability of learning for a specific infinite-horizon objective, arguing for the use of a discounted variant of LTL. 

A similar issue exists for the infinite-horizon, average-reward objectives.
In particular, it is understood that \reinforcementlearning algorithms do not have guarantees on the learned behavior for infinite\Hyphdash horizon, average\Hyphdash reward problems without additional assumptions on the environment~\citep{Kearns2002NearOptimalRL}.

However, to our knowledge, no prior work has formally analyzed the learnability of LTL objectives.\footnote{
Concurrent to this work, \citet{alur2021framework} also examine the intractability of LTL objectives.
They state and prove a theorem that is a weaker version of the core theorem of this work. 
Their work was made public while this work was under conference review.
We discuss their work in \Cref{sec:concurrent_work}.}

\paragraph{Our Results}
We leverage the PAC-MDP framework~\citep{pacmodelfreerl} to prove that reinforcement learning for infinite\Hyphdash horizon LTL objectives is intractable.
The intuition for this intractability is: 
Any finite number of interactions with an environment with unknown transition dynamics is insufficient to identify the environment dynamics perfectly.
Moreover, for an infinite-horizon objective, a behavior's satisfaction probability under the inaccurate environment dynamics can be arbitrarily different from the behavior's satisfaction probability under the true dynamics.
Consequently, a learner cannot guarantee with any confidence that it has identified near-optimal behavior for an infinite-horizon objective.

\subsection{Implications for Relevant and Future Work}

Our results provide a framework to categorize approaches that either focus on tractable LTL objectives or weaken the guarantees of an algorithm.
As a result, we interpret several previous approaches as instantiations of the following categories:
\ifshorterpaper
\begin{enumerate*}
\item works with tractable LTL objectives that are not infinite-horizon,
\item provides a best-effort guarantee,
\item requires assumptions on the environment, or
\item works with LTL-like objective with a less demanding semantics.
\end{enumerate*} %
We discuss our categorization and review these approaches in \Cref{sec:directions_forward}.
\else
\begin{itemize}[leftmargin=*,wide=0pt]
\item 
Work with finite-horizon LTL objectives, the complement of infinite\Hyphdash horizon objectives, to obtain guarantees on the learned behavior \citep{smcbltl}. These objectives, like $a \land \ltlnext a$ ($a$ is true for two steps), 
are decidable within a known finite number of steps.
\item 
Seek a best-effort confidence interval \citep{pacmdpsmcconfidenceinterval}.
Specifically, the interval can be trivial in the worst case, denoting that learned behavior is a maximally poor approximation of the optimal behavior.
\item
Make additional assumptions about the environment to obtain guarantees on the learned behavior \citep{ufukpacmdpltl,smcpmin}.
\item 
Change the problem by working with LTL-like objectives such as:
\begin{enumerate*}
\item relaxed LTL objectives that become exactly LTL in the (unreachable) limit~\citep{dorsa,omegaregularrl19,hasanbeig2019reinforcement,bozkurt2020control} and
\item objectives that use temporal operators but employ a different semantics~\citep{gltl17,truncatedltl-rl,ltlf-rl,rewardmachine1}.
\end{enumerate*}
The learnability of these objectives is a potential future research direction.
\end{itemize}
\fi

\subsection{Contributions}

We make the following contributions:
\begin{itemize}[leftmargin=*, wide=0pt]
\item A formalization of reinforcement learning with LTL objectives under the probably approximately correct in Markov decision processes (PAC-MDP) framework \citep{Fiechter94efficientreinforcement,Kearns2002NearOptimalRL,kakade03}, a standard framework for measuring sample complexity for \reinforcementlearning algorithms; 
and a formal definition of LTL-PAC-learnable, a learnability criterion for LTL objectives.
\item A statement and proof that:
\begin{enumerate*}
\item Any infinite-horizon LTL formula is not LTL-PAC-learnable.
\item Any finite\Hyphdash horizon LTL formula is LTL-PAC-learnable.
\end{enumerate*}
To that end, for any infinite-horizon formula, we give a construction of two special families of MDPs as counterexamples with which we prove that the formula is not LTL-PAC-learnable. 
\item Experiments with current \reinforcementlearning algorithms for LTL objectives that provide empirical support for our theoretical result.
\item A categorization of approaches that focus on tractable objectives or weaken the guarantees of LTL-PAC-learnable and a classification of previous approaches into these categories. 
\end{itemize}

\section{Preliminaries: Reinforcement Learning}
\label{sec:preliminaries}

This section provides definitions for MDPs, planning, reinforcement learning, and PAC-MDP.

\subsection{Markov Processes}
\label{sec:background_markov_processes}
\ifshorterpaper\else
We first review some basic notations for Markov processes.
\fi

A {\em Markov decision process} (MDP) is a tuple $\mathcal{M} = (S,A,P,s_0)$, where $S$ and $A$ are finite sets of states and actions, $\functiontype[P]{S, A}{\distributionon{S}}$ is a transition probability function that maps a current state and an action to a distribution over next states, and $s_0 \in S$ is an initial state. The MDP is sometimes referred to as the \emph{environment MDP} to distinguish it from any specific objective.

A (stochastic) {\em Markovian policy} $\pi$ for an MDP is a function $\functiontype[\pi]{S}{\distributionon{A}}$ that maps each state of the MDP to a distribution over the actions.

A (stochastic) {\em non-Markovian policy} $\pi$ for an MDP is a function $\functiontype[\pi]{\left(S \times A\right)^*, S}{\distributionon{A}}$ that maps a history of states and actions of the MDP to a distribution over actions.

An MDP and a policy on the MDP induce a {\em discrete-time Markov chain} (DTMC). 
A DTMC is a tuple $\mathcal{D} = (S, P, s_0)$, where $S$ is a finite set of states, $\functiontype[P]{S}{\distributionon{S}}$ is a transition-probability function that maps a current state to a distribution over next states, and $s_0 \in S$ is an initial state.
A {\em sample path} of $\mathcal{D}$ is an infinite sequence of states $w \in \infseq{S}$.
The sample paths of a DTMC form a probability space.

\subsection{Objective}

An {\em objective} for an MDP $\mathcal{M} = (S, A, P, s_0)$  is a measurable function $\functiontype[\kappa]{\infseq{S}}{\reals}$ on the probability space of the DTMC $\mathcal{D}$ induced by $\mathcal{M}$ and a policy $\pi$.
The {\em value} of the objective for the MDP $\mathcal{M}$ and a policy $\pi$ is the expectation of the objective under that probability space:
\begin{equation*}
\mdpvaluefunc{\pi}{\mathcal{M}}{\kappa} = \expectation*{w \sim \mathcal{D}}{\kappa(w)}
\quad (\mathcal{D} \text{ induced by } \mathcal{M} \text{ and } \pi).
\end{equation*}

For example, the cumulative discounted rewards objective \citep{puterman94} with discount $\gamma$ and a reward function $\functiontype[R]{S}{\reals}$ is: $
\kappa^\text{reward}(w) \triangleq \sum_{i=0}^{\infty} \gamma^i \cdot R(\windex{w}{i}).
$

An {\em optimal policy} maximizes the objective's value:
$
\pi^* = \argmax_{\pi} \mdpvaluefunc{\pi}{\mathcal{M}}{\kappa}
$.
The {\em optimal value} $\mdpvaluefunc{\pi^*}{\mathcal{M}}{\kappa}$ is then the objective value of the optimal policy.
A policy $\pi$ is {\em $\epsilon$-optimal} if its value is $\epsilon$-close to the optimal value: $\mdpvaluefunc{\pi}{\mathcal{M}}{\kappa} \ge \mdpvaluefunc{\pi^*}{\mathcal{M}}{\kappa} - \epsilon$.

\subsection{Planning with a Generative Model}

A \planningwithgenerativemodel algorithm \citep{approximateplanninginpompdps,Grill2016BlazingTT} has access to a {\em generative model}, a sampler, of an MDP's transitions but does not have direct access to the underlying probability values. It can take any state and action and sample a next state. It learns a policy from those sampled transitions.

Formally, a \planningwithgenerativemodel algorithm $\mathcal{A}$ is a tuple $(\mathcal{A}^\text{S}, \mathcal{A}^\text{L})$, where $\mathcal{A}^\text{S}$ is a {\em sampling algorithm} that drives how the environment is sampled, and $\mathcal{A}^\text{L}$ is a {\em learning algorithm} that learns a policy from the samples obtained by applying the sampling algorithm.

In particular, the sampling algorithm $\mathcal{A}^\text{S}$ is a function that maps from a history of sampled environment transitions \begingroup\small $((s_0, a_0, s'_0) \dots (s_k, a_k, s'_k))$ \endgroup to the next state and action to sample \begingroup\small$(s_{k+1}, a_{k+1})$ \endgroup, resulting in 
\begingroup\small$\rv{s}_{k+1}' \sim \prob*{}{ \cdot | s_{k+1}, a_{k+1}}$\endgroup.
Iterative application of the sampling algorithm $\mathcal{A}^\text{S}$ produces a sequence of sampled environment transitions.

The learning algorithm is a function that maps that sequence of sampled environment transitions to a non\Hyphdash Markovian policy of the environment MDP. 
Note that the sampling algorithm can internally consider alternative policies as part of its decision of what to sample.
Also, note that we deliberately consider non-Markovian policies since the optimal policy for an LTL objective (defined later) is non-Markovian in general (unlike a cumulative discounted rewards objective).

\subsection{Reinforcement Learning}

In reinforcement learning, an agent is situated in an environment MDP and only observes state transitions.
We also allow the agent to reset to the initial state as in \citet{Fiechter94efficientreinforcement}.

We can view a \reinforcementlearning algorithm as a special kind of \planningwithgenerativemodel algorithm 
$(\mathcal{A}^\text{S}, \mathcal{A}^\text{L})$ such that the sampling algorithm always either follows the next state sampled from the environment or resets to the initial state of the environment.

\subsection{Probably Approximately Correct in MDPs}

A successful \planningwithgenerativemodel algorithm (or \reinforcementlearning algorithm) should learn from the sampled environment transitions and produce an optimal policy for the objective in the environment MDP.
However, since the environment transitions may be stochastic, we cannot expect an algorithm to always produce the optimal policy.
Instead, we seek an algorithm that, with high probability, produces a nearly optimal policy.
The PAC-MDP framework~\citep{Fiechter94efficientreinforcement,Kearns2002NearOptimalRL,kakade03}, which takes inspiration from probably approximately correct (PAC) learning~\citep{valiantpaclearnable}, formalizes this notion.
The PAC-MDP framework requires efficiency in both sampling and algorithmic complexity.
In this work, we only consider sample efficiency and thus omit the requirement on algorithmic complexity. 
Next, we generalize the PAC-MDP framework from \reinforcementlearning with a reward objective to \planningwithgenerativemodel with a generic objective.

\begin{definition}
\label{def:rl-obj-pac-algo}
Given an objective $\kappa$, a \planningwithgenerativemodel algorithm  $(\mathcal{A}^\text{S}, \mathcal{A}^\text{L})$ is {\em $\kappa$-PAC} (probably approximately correct for objective $\kappa$) in an environment MDP $\mathcal{M}$ if, with the sequence of transitions $T$ of length $N$ sampled using the sampling algorithm $\mathcal{A}^\text{S}$, the learning algorithm $\mathcal{A}^\text{L}$ outputs a non-Markovian $\epsilon$-optimal policy with probability at least $1-\delta$ for any given $\epsilon>0$ and $0 < \delta < 1$. That is:
\begin{equation}
\label{eq:pac-obj-epsilon-delta-inequality}
\prob*{\rv{T} \sim \mdpsamplingproduct{\mathcal{M}}{\mathcal{A}^\text{S}}{N}}{\mdpvaluefunc{\mathcal{A}^\text{L}\left(T\right)}{\mathcal{M}}{\kappa} \ge \mdpvaluefunc{\pi^*}{\mathcal{M}}{\kappa} - \epsilon} \ge 1 - \delta.
\end{equation}
\end{definition}
We use \begingroup\small$ \rv{T}{\sim}\mdpsamplingproduct{\mathcal{M}}{\mathcal{A}^\text{S}}{N}$\endgroup to denote that the probability space is over the set of length-$N$ transition sequences sampled from the environment $\mathcal{M}$ using the sampling algorithm $\mathcal{A}^\text{S}$.
For brevity, we will drop $\mdpsamplingproduct{\mathcal{M}}{\mathcal{A}^\text{S}}{N}$ when it is clear from context and simply write $\prob{\rv{T}}{.}$ to denote that the probability space is over the sampled transitions.

\begin{definition}
\label{def:rl-obj-sample-efficiently-pac-algo}
Given an objective $\kappa$, a $\kappa$-PAC \planningwithgenerativemodel algorithm  is {\em sample efficiently $\kappa$-PAC} if the number of sampled transitions $N$ is asymptotically polynomial in $\frac{1}{\epsilon}$, $\frac{1}{\delta}$, $|S|$, $|A|$.
\end{definition}
Note that the definition allows the polynomial to have constant coefficients that depends on $\kappa$. 

\section{Linear Temporal Logic Objectives}

This section describes LTL and its use in objectives.

\subsection{Linear Temporal Logic}
\label{sec:background_ltl}

A linear temporal logic (LTL) formula is built from a finite set of atomic propositions $\Pi$, logical connectives $\neg, \land, \lor$, temporal next $\ltlnext$, and temporal operators  $\always$ (always), $\eventually$ (eventually), and $\until$ (until). 
\Cref{eq:ltl_grammar} gives the grammar of an LTL formula $\phi$ over the set of atomic propositions $\Pi$:
\begingroup
\newcommand{\alt}{\:\big|\:}%
\begin{equation}
\label{eq:ltl_grammar}
\phi \defeq 
a
\alt \neg \phi 
\alt \phi \land \phi 
\alt \phi \lor \phi
\alt \ltlnext \phi
\alt \always \phi 
\alt \eventually \phi
\alt \phi \until \phi, \; a \in \Pi.
\end{equation}
\endgroup
LTL is a logic over infinite-length words. 
Informally, these temporal operators have the following meanings:
$\ltlnext \phi$ asserts that $\phi$ is true at the next time step;
$\always \phi$ asserts that $\phi$ is always true;
$\eventually \phi$ asserts that $\phi$ is eventually true;
$\psi \until \phi$ asserts that $\psi$ needs to stay true until $\phi$ eventually becomes true.
We give the formal semantics of each operator in \Cref{sec:ltl-hierarchy-full}.
We write $w \vDash \phi$ to denote that the infinite word $w$ satisfies $\phi$.

\subsection{MDP with LTL Objectives}
\label{sec:ltl-objectives}

An LTL objective maximizes the probability of satisfying an LTL formula. We formalize this notion below.

An {\em LTL specification} for an MDP is a tuple $(\mathcal{L}, \phi)$, where $\functiontype[\mathcal{L}]{S}{2^\Pi}$ is a labeling function, and $\phi$ is an LTL formula over atomic propositions $\Pi$.
The labeling function is a classifier mapping each MDP state to a tuple of truth values of the atomic propositions in $\phi$.
For a sample path $w$, we use $\wmap{\mathcal{L}}{w}$ to denote the element-wise application of $\mathcal{L}$ on $w$.

The LTL objective $\xi$ specified by the LTL specification is the satisfaction of the formula $\phi$ of a sample path mapped by the labeling function $\mathcal{L}$, that is: $\kappa(w) \triangleq \indicator{\mathcal{L}(w) \vDash \phi}$.
The value of this objective is called the {\em satisfaction probability} of $\xi$:  \begin{equation*}
\mdpvaluefunc{\pi}{\mathcal{M}}{\xi} = \prob*{w \sim \mathcal{D}}{\wmap{\mathcal{L}}{w} \vDash \phi} \quad (\mathcal{D} \text{ induced by } \mathcal{M} \text{ and } \pi).
\end{equation*}

\subsection{Infinite Horizons in LTL Objectives}

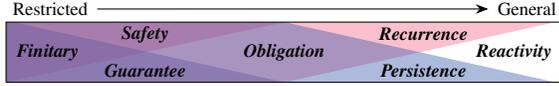
\begin{figure}[tb]
\centering
\tikzset{external/export next=false}
\begin{tikzpicture}
\def\baseheight{0.1}
\def\levelwidth{0.8};
\def\levelheight{3.6};
\definecolor{leftcolor1}{HTML}{F76385}
\definecolor{leftcolor2}{HTML}{F76385}

\definecolor{rightcolor1}{HTML}{3053A9}
\definecolor{rightcolor2}{HTML}{3053A9}

\draw[] (0,0)--({2*\levelheight+2*\baseheight},0)--(2*\levelheight+2*\baseheight,\levelwidth)--(0,\levelwidth)--(0,0);
\begin{scope}[on background layer]
    \fill[fill=leftcolor1, fill opacity=0.4] (0,0)--(0,\levelwidth)--(\baseheight,\levelwidth)-- (\baseheight+\levelheight,\levelwidth);
    \fill[fill=rightcolor1, fill opacity=0.4] (0,0)--(\baseheight+\levelheight,0)--(\baseheight,\levelwidth)--(0,\levelwidth);
    \fill[fill=leftcolor2, fill opacity=0.4] (0,0)--(0,\levelwidth)--(\baseheight+2*\levelheight,\levelwidth)--(\baseheight+\levelheight,0);
    \fill[fill=rightcolor2, fill opacity=0.4] (0,0)--(\baseheight+2*\levelheight,0)--(\baseheight+\levelheight,\levelwidth)--(0,\levelwidth);
\end{scope}

\node[] (bottom) at ({(\baseheight+\levelheight) * 0.6/4}, \levelwidth/2) {\scriptsize \ltlclass*{finitary}};
\node[] (guarantee) at ({(\baseheight+\levelheight)/2}, \levelwidth * 1/5) {\scriptsize \ltlclass*{guarantee}};
\node[] (safety) at ({(\baseheight+\levelheight)/2}, {\levelwidth * (1-1/5)}) {\scriptsize \ltlclass*{safety}};
\node[] (obligation) at ({\baseheight+\levelheight}, \levelwidth/2) {\scriptsize \ltlclass*{obligation}};
\node[] (persistence) at ({(\baseheight+\levelheight) * 3/2}, \levelwidth * 1/5) {\scriptsize \ltlclass*{persistence}};
\node[] (recurrence) at ({(\baseheight+\levelheight) * 3/2}, {\levelwidth * (1-1/5)}) {\scriptsize \ltlclass*{recurrence}};
\node[] (reactivity) at ({(\baseheight+\levelheight) * 7.3/4}, \levelwidth/2) {\scriptsize \ltlclass*{reactivity}};

\node[above right=-0.1em and -0.1em] (label_restricted) at (0,\levelwidth) {\scriptsize Restricted} ;

\node[above left=-0.1em and -0.1em] (label_general) at ({2*(\baseheight+\levelheight)},\levelwidth) {\scriptsize General} ;

\path[->] (label_restricted) edge [style=-{Stealth}] (label_general);

\end{tikzpicture}
\caption{The hierarchy of LTL}
\label{fig:ltl-hierarchy}
\end{figure}

An LTL formula describes either a finite-horizon or infinite\Hyphdash horizon property.
\citet{tlhierarchy} classified LTL formulas into seven classes, as shown in \Cref{fig:ltl-hierarchy}.
Each class includes all the classes to the left of that class (e.g., $\ltlclass{finitary}\,{\subset}\,\ltlclass{guarantee}$, but $\ltlclass{safety} \,{\not\subset}\,\ltlclass{guarantee}$), with the $\ltlclass{finitary}$ class being the most restricted and the $\ltlclass{reactivity}$ class being the most general.
Below we briefly describe the key properties of the leftmost three classes relevant to the core of this paper. We present a complete description of all the classes in \Cref{sec:ltl-hierarchy-full}.
\begin{itemize}[leftmargin=*]
\item $\phi\inltlclass{finitary}$ iff there exists a horizon $H$ such that infinite-length words sharing the same prefix of length $H$ are either all accepted or all rejected by $\phi$. E.g., $a \land \ltlnext a$ (i.e., $a$ is true for two steps) is in $\ltlclass{finitary}$. 
\item $\phi\inltlclass{guarantee}$ iff there exists a language of finite words $L$ (i.e., a Boolean function on finite-length words) such that $w \vDash \phi$ if $L$ accepts a prefix of $w$. Informally, a formula in $\ltlclass{guarantee}$ asserts that something eventually happens. E.g., $\eventually a$ (i.e., eventually $a$ is true) is in $\ltlclass{guarantee}$.
\item $\phi\inltlclass{safety}$ iff there exists a language of finite words $L$ such that $w \vDash \phi$ if $L$ accepts all prefixes of $w$. Informally, a formula in $\ltlclass{safety}$ asserts that something always happens. E.g., $\always a$ (i.e., $a$ is always true) is in $\ltlclass{safety}$.
\end{itemize}

Moreover, the set of \ltlclasst{finitary} is the intersection of the set of \ltlclasst{guarantee} formulas and the set of \ltlclasst{safety} formulas.
Any $\phi\inltlclass{finitary}$, or equivalently $\phi\inltlclass{guarantee} \cap \ltlclass{safety}$, inherently describes finite-horizon properties.
Any $\phi\notinltlclass{finitary}$, or equivalently $\phi\inltlclass{guarantee}^\complement \cup \ltlclass{safety}^\complement$, inherently describes infinite-horizon properties.
We will show that \reinforcementlearning algorithms cannot provide PAC guarantees for LTL objectives specified by formulas that describe infinite\Hyphdash horizon properties.

\subsection{Intuition of the Problem}

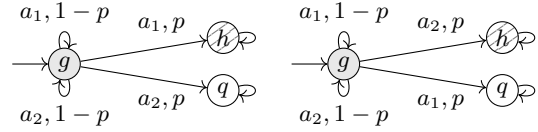
\begin{figure}[tb]
\centering
\tikzset{external/export next=false}
\newcommand{\counterexamplemdpkind}{m1}
\begin{tikzpicture}[node distance=1.5cm,on grid,auto]
\tikzset{%
    in place/.style={
      auto=false,
      fill=white,
      inner sep=2pt,
    },
}

\newcommand{\statenodesize}{1.2em}

\tikzset{%
    every state/.style={
        fill={rgb:black,1;white,10},
        initial text=, 
        inner sep=0, 
        minimum size=\statenodesize,
        font={\small},
    },
    tl/.style={font=\small} %
}

\tikzset{%
    accstate/.style={
        state,
        pattern={Lines[angle=35,distance={4.5pt/sqrt(2)}]},
        pattern color=gray
    }
}

\tikzset{%
    rejstate/.style={state, fill=white}
}

\ifdefstring{\counterexamplemdpkind}{m12}{
\tikzset{%
    a1/.style={
      swap,
      auto=left,
      to path={ let \p1=(\tikztostart),\p2=(\tikztotarget), \n1={atan2(\y2-\y1,\x2-\x1)},\n2={\n1+180} in ($(\tikztostart.{\n1})!0.5mm!90:(\tikztotarget.{\n2})$) -- ($(\tikztotarget.{\n2})!0.5mm!270:(\tikztostart.{\n1})$) \tikztonodes},
    },
    a2/.style={
      auto=right,
      to path={ let \p1=(\tikztostart),\p2=(\tikztotarget), \n1={atan2(\y2-\y1,\x2-\x1)},\n2={\n1+180} in ($(\tikztostart.{\n1})!0.5mm!270:(\tikztotarget.{\n2})$) -- ($(\tikztotarget.{\n2})!0.5mm!90:(\tikztostart.{\n1})$) \tikztonodes}
    },
}
}

\ifdefstring{\counterexamplemdpkind}{m1}{%
\tikzset{
    a2/.style={draw=none, fill opacity=0},
    a1/.style={swap, auto=left}
}
}

\ifdefstring{\counterexamplemdpkind}{m2}{%
\tikzset{
    a2/.style={auto=right},
    a1/.style={draw=none, fill opacity=0}
}
}

\providecommand{\arrowheadscale}{0.8}
\tikzset{%
    m1/.style={
        -{Stealth[scale=\arrowheadscale]}
    },
    m2/.style={
        -{Triangle[open, scale=\arrowheadscale]}
    },
}

\ifdefstring{\counterexamplemdpkind}{m12}%
{
\tikzset{
    m12/.style={
        -{Stealth[scale=\arrowheadscale] Triangle[open, scale=\arrowheadscale]}
    }
}
}

\ifdefstring{\counterexamplemdpkind}{m1}%
{\tikzset{
    m12/.style={m1}
}}

\ifdefstring{\counterexamplemdpkind}{m2}%
{\tikzset{
    m12/.style={m2}
}}

\newcommand{\largedots}{$\textbf{\ldots}$}

    \tikzstyle{every state}=[
        fill={rgb:black,1;white,10},
        initial text=, 
        inner sep=0, 
        minimum size=1.2em,
        font={\small},
    ]
    
    \tikzstyle{tl/.style}=[font=\small] %

    \tikzstyle{every loop}=[
    style={
        looseness=1, 
        min distance=3mm,
        font={\small}
        }
    ]
    
    \node[state, initial left] (g) {$g$};
    \node[accstate] (h) [above right =1em and 6em of g]  {$h$};
    \node[rejstate] (q) [below right =1em and 6em of g] {$q$};

    \ifdefstring{\counterexamplemdpkind}{m1}{%
    \path[->]
    (g) edge [pos=0.9] node [tl] {$a_1,p$} (h)
    (g) edge [swap, pos=0.9] node [tl] {$a_2,p$} (q)
    ;
    }
    
    \ifdefstring{\counterexamplemdpkind}{m2}{%
    \path[->]
    (g) edge [pos=0.9] node [tl] {$a_2,p$} (h)
    (g) edge [swap, pos=0.9] node [tl] {$a_1,p$} (q)
    ;
    }
    
    \path[->] 
    (g) edge  [loop above]  node [tl] {$a_1,1 - p$}  ()
    (g) edge  [loop below]  node [tl] {$a_2,1 - p$}  ()
    
    (h) edge  [loop right]  node {}  ()
    (q) edge  [loop right]  node {}  ()
    
    ;
\end{tikzpicture}
\unskip%
\renewcommand{\counterexamplemdpkind}{m2}%
\begin{tikzpicture}[node distance=1.5cm,on grid,auto]
\tikzset{%
    in place/.style={
      auto=false,
      fill=white,
      inner sep=2pt,
    },
}

\newcommand{\statenodesize}{1.2em}

\tikzset{%
    every state/.style={
        fill={rgb:black,1;white,10},
        initial text=, 
        inner sep=0, 
        minimum size=\statenodesize,
        font={\small},
    },
    tl/.style={font=\small} %
}

\tikzset{%
    accstate/.style={
        state,
        pattern={Lines[angle=35,distance={4.5pt/sqrt(2)}]},
        pattern color=gray
    }
}

\tikzset{%
    rejstate/.style={state, fill=white}
}

\ifdefstring{\counterexamplemdpkind}{m12}{
\tikzset{%
    a1/.style={
      swap,
      auto=left,
      to path={ let \p1=(\tikztostart),\p2=(\tikztotarget), \n1={atan2(\y2-\y1,\x2-\x1)},\n2={\n1+180} in ($(\tikztostart.{\n1})!0.5mm!90:(\tikztotarget.{\n2})$) -- ($(\tikztotarget.{\n2})!0.5mm!270:(\tikztostart.{\n1})$) \tikztonodes},
    },
    a2/.style={
      auto=right,
      to path={ let \p1=(\tikztostart),\p2=(\tikztotarget), \n1={atan2(\y2-\y1,\x2-\x1)},\n2={\n1+180} in ($(\tikztostart.{\n1})!0.5mm!270:(\tikztotarget.{\n2})$) -- ($(\tikztotarget.{\n2})!0.5mm!90:(\tikztostart.{\n1})$) \tikztonodes}
    },
}
}

\ifdefstring{\counterexamplemdpkind}{m1}{%
\tikzset{
    a2/.style={draw=none, fill opacity=0},
    a1/.style={swap, auto=left}
}
}

\ifdefstring{\counterexamplemdpkind}{m2}{%
\tikzset{
    a2/.style={auto=right},
    a1/.style={draw=none, fill opacity=0}
}
}

\providecommand{\arrowheadscale}{0.8}
\tikzset{%
    m1/.style={
        -{Stealth[scale=\arrowheadscale]}
    },
    m2/.style={
        -{Triangle[open, scale=\arrowheadscale]}
    },
}

\ifdefstring{\counterexamplemdpkind}{m12}%
{
\tikzset{
    m12/.style={
        -{Stealth[scale=\arrowheadscale] Triangle[open, scale=\arrowheadscale]}
    }
}
}

\ifdefstring{\counterexamplemdpkind}{m1}%
{\tikzset{
    m12/.style={m1}
}}

\ifdefstring{\counterexamplemdpkind}{m2}%
{\tikzset{
    m12/.style={m2}
}}

\newcommand{\largedots}{$\textbf{\ldots}$}

    \tikzstyle{every state}=[
        fill={rgb:black,1;white,10},
        initial text=, 
        inner sep=0, 
        minimum size=1.2em,
        font={\small},
    ]
    
    \tikzstyle{tl/.style}=[font=\small] %

    \tikzstyle{every loop}=[
    style={
        looseness=1, 
        min distance=3mm,
        font={\small}
        }
    ]
    
    \node[state, initial left] (g) {$g$};
    \node[accstate] (h) [above right =1em and 6em of g]  {$h$};
    \node[rejstate] (q) [below right =1em and 6em of g] {$q$};

    \ifdefstring{\counterexamplemdpkind}{m1}{%
    \path[->]
    (g) edge [pos=0.9] node [tl] {$a_1,p$} (h)
    (g) edge [swap, pos=0.9] node [tl] {$a_2,p$} (q)
    ;
    }
    
    \ifdefstring{\counterexamplemdpkind}{m2}{%
    \path[->]
    (g) edge [pos=0.9] node [tl] {$a_2,p$} (h)
    (g) edge [swap, pos=0.9] node [tl] {$a_1,p$} (q)
    ;
    }
    
    \path[->] 
    (g) edge  [loop above]  node [tl] {$a_1,1 - p$}  ()
    (g) edge  [loop below]  node [tl] {$a_2,1 - p$}  ()
    
    (h) edge  [loop right]  node {}  ()
    (q) edge  [loop right]  node {}  ()
    
    ;

\end{tikzpicture}
\caption{Two MDPs parameterized by $p$ in range $0 < p < 1$. 
Action $a_1$ in the MDP on the left and action $a_2$ in the MDP on the right have probability $p$ of transitioning to the state $h$. 
Conversely, action $a_2$ in the MDP on the left and action $a_1$ in the MDP on the right have probability $p$ of transitioning to the state $q$. 
Both actions in both MDPs have probability $1 - p$ to loop around the state $g$.
}
\label{fig:counterexample-mdp-simple}
\end{figure}

Suppose that we send an agent into one of the MDPs in \Cref{fig:counterexample-mdp-simple}, and want its behavior to satisfy ``eventually reach the state $h$'', expressed as the LTL formula $\eventually h$.
The optimal behavior is to always choose the action along the transition $g\rightarrow h$ for both MDPs (i.e., $a_1$ for the MDP on the left and $a_2$ for the MDP on the right).
This optimal behavior satisfies the objective with probability one.
However, the agent does not know which of the two MDPs it is in.
The agent must follow its sampling algorithm to explore the MDP's dynamics and use its learning algorithm to learn this optimal behavior.

If the agent observes neither transitions going out of $g$ (i.e., $g\rightarrow h$ or $g\rightarrow q$) during sampling,
it will not be able to distinguish between the two actions.
The best it can do is a 50\% chance guess and cannot provide any non-trivial guarantee on the probability of learning the optimal action.

On the other hand, if the agent observes one of the transitions going out of $g$, it will be able to determine which action leads to state $h$, thereby learning always to take that action.
However, the probability of observing any such transition with $N$ interactions is at most $1 - (1 - p)^N$.
This is problematic: with any finite $N$, there always exists a value of $p$ such that this probability is arbitrarily close to $0$. 
In other words, with any finite number of interactions, without knowing the value of $p$, the agent cannot guarantee (a non-zero lower bound on) its chance of learning a policy that satisfies the LTL formula $\eventually h$.

Further, the problem is not limited to this formula.
For example, the objective ``never reach the state $q$'', expressed as the formula $\always \neg q$, has the same problem in these two MDPs.
More generally, for any LTL formula describing an infinite-horizon property, we construct two counterexample MDPs with the same nature as the ones in \Cref{fig:counterexample-mdp-simple}, and prove that it is impossible to guarantee learning the optimal policy.

\section{Learnability of LTL Objectives}
\label{sec:maintheoremstatement}

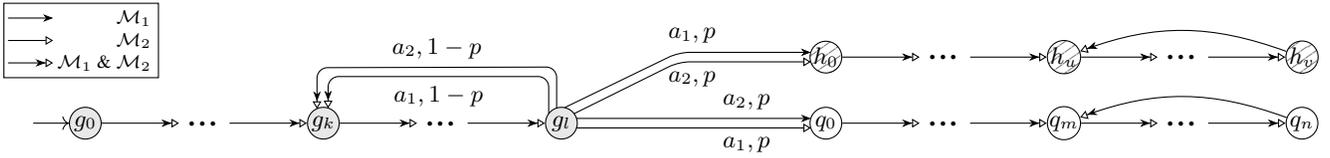
\begin{figure*}[tb]
\centering
\newcommand{\counterexamplemdpkind}{m12}
\begin{tikzpicture}[node distance=4.5em,on grid,auto]

\ifdefstring{\counterexamplemdpkind}{m12}{
\newcommand{\counterexamplemdplegend}{%
\path ([xshift=2.5em,yshift=-0.8em] current bounding box.north west)%
 node[matrix, cells={nodes={font={\scriptsize}}, anchor=east}, draw,inner sep=0.3ex, ampersand replacement={\&}]%
{%
  \draw[m1] (0,0) -- ++ (0.6,0) {}; \& \node[] {$\mathcal{M}_1$}; \\%
  \draw[m2] (0,0) -- ++ (0.6,0) {}; \& \node[] {$\mathcal{M}_2$};\\%
  \draw[m12] (0,0) -- ++ (0.6,0) {}; \& \node[] {$\mathcal{M}_1$ \string& $\mathcal{M}_2$}; \\%
};
}}{
\newcommand{\counterexamplemdplegend}{}
}

    \tikzpicturedependsonfile{diagrams/counterexample-mdp-common.tex}
    \tikzset{%
    in place/.style={
      auto=false,
      fill=white,
      inner sep=2pt,
    },
}

\newcommand{\statenodesize}{1.2em}

\tikzset{%
    every state/.style={
        fill={rgb:black,1;white,10},
        initial text=, 
        inner sep=0, 
        minimum size=\statenodesize,
        font={\small},
    },
    tl/.style={font=\small} %
}

\tikzset{%
    accstate/.style={
        state,
        pattern={Lines[angle=35,distance={4.5pt/sqrt(2)}]},
        pattern color=gray
    }
}

\tikzset{%
    rejstate/.style={state, fill=white}
}

\ifdefstring{\counterexamplemdpkind}{m12}{
\tikzset{%
    a1/.style={
      swap,
      auto=left,
      to path={ let \p1=(\tikztostart),\p2=(\tikztotarget), \n1={atan2(\y2-\y1,\x2-\x1)},\n2={\n1+180} in ($(\tikztostart.{\n1})!0.5mm!90:(\tikztotarget.{\n2})$) -- ($(\tikztotarget.{\n2})!0.5mm!270:(\tikztostart.{\n1})$) \tikztonodes},
    },
    a2/.style={
      auto=right,
      to path={ let \p1=(\tikztostart),\p2=(\tikztotarget), \n1={atan2(\y2-\y1,\x2-\x1)},\n2={\n1+180} in ($(\tikztostart.{\n1})!0.5mm!270:(\tikztotarget.{\n2})$) -- ($(\tikztotarget.{\n2})!0.5mm!90:(\tikztostart.{\n1})$) \tikztonodes}
    },
}
}

\ifdefstring{\counterexamplemdpkind}{m1}{%
\tikzset{
    a2/.style={draw=none, fill opacity=0},
    a1/.style={swap, auto=left}
}
}

\ifdefstring{\counterexamplemdpkind}{m2}{%
\tikzset{
    a2/.style={auto=right},
    a1/.style={draw=none, fill opacity=0}
}
}

\providecommand{\arrowheadscale}{0.8}
\tikzset{%
    m1/.style={
        -{Stealth[scale=\arrowheadscale]}
    },
    m2/.style={
        -{Triangle[open, scale=\arrowheadscale]}
    },
}

\ifdefstring{\counterexamplemdpkind}{m12}%
{
\tikzset{
    m12/.style={
        -{Stealth[scale=\arrowheadscale] Triangle[open, scale=\arrowheadscale]}
    }
}
}

\ifdefstring{\counterexamplemdpkind}{m1}%
{\tikzset{
    m12/.style={m1}
}}

\ifdefstring{\counterexamplemdpkind}{m2}%
{\tikzset{
    m12/.style={m2}
}}

\newcommand{\largedots}{$\textbf{\ldots}$}

    \node[state,initial left]   (g_0)                      {$g_0$};
    \node[] (dots_g_1_g_k)  [right =of g_0] {\largedots};
    \node[state] (g_k)  [right =of dots_g_1_g_k] {$g_k$};
    \node[] (dots_g_k_g_l)  [right =of g_k] {\largedots};
    \node[state] (g_l)  [right =of dots_g_k_g_l] {$g_l$};

    \path[->]
    (g_0) edge [m12] (dots_g_1_g_k)
    (dots_g_1_g_k) edge [m12] (g_k)

    (g_k) edge [m12] (dots_g_k_g_l)
    (dots_g_k_g_l) edge [m12] (g_l)

    (g_l) edge [rounded corners, to path={
        ($ (g_l) + (-0.4*\statenodesize, 0.3*\statenodesize) $)
        -- ($ (g_l) + (-0.4*\statenodesize, 0.47*\statenodesize) + (0, \statenodesize) $)
        -- ($ (g_k) + (0.15*\statenodesize, 0.47*\statenodesize) + (0, \statenodesize) $) \tikztonodes -- ($ (g_k) + (0.15*\statenodesize, 0.47*\statenodesize)$)
        }, m12] node [tl,midway] {$a_1, 1 - p$} (g_k)

    (g_l) edge [rounded corners, to path={
        ($ (g_l) + (-0.15*\statenodesize, 0.47*\statenodesize) $)
        -- ($ (g_l) + (-0.15*\statenodesize, 0.47*\statenodesize) + (0, 1.3*\statenodesize) $)
        -- ($ (g_k) + (-0.2*\statenodesize, 0.45*\statenodesize) + (0, 1.3*\statenodesize) $) \tikztonodes -- ($ (g_k) + (-0.2*\statenodesize, 0.45*\statenodesize)$)
        }, swap, m12] node [tl,midway] {$a_2, 1 - p$} (g_k);
    
    \node[accstate] (h_0)  [above right=2.5em and 10em of g_l] {$h_0$};
    \node[] (dots_h_0_h_u)  [right =of h_0] {\largedots};
    \node[accstate] (h_u)  [right =of dots_h_0_h_u] {$h_u$};
    \node[] (dots_h_u_h_v)  [right =of h_u] {\largedots};
    \node[accstate] (h_v)  [right =of dots_h_u_h_v] {$h_v$};
    
    \node[rejstate] (q_0)  [right=10em of g_l] {$q_0$};
    \node[] (dots_q_0_q_m)  [right =of q_0] {\largedots};
    \node[rejstate] (q_m)  [right =of dots_q_0_q_m] {$q_m$};
    \node[] (dots_q_m_q_n)  [right =of q_m] {\largedots};
    \node[rejstate] (q_n)  [right =of dots_q_m_q_n] {$q_n$};

    \path[->]

    (g_l) edge [rounded corners, to path={
        ($ (g_l) + (0.1*\statenodesize, 0.49*\statenodesize) $) -- ($ (h_0) + (-0.47*\statenodesize, 0.15*\statenodesize) + (-5em, 0) $)  -- +(5em, 0) \tikztonodes
        }, m1] node [tl, xshift=-2em] {$a_1, p$} (h_0)

    (g_l) edge [rounded corners, to path={
        ($ (g_l) + (0.4*\statenodesize, 0.3*\statenodesize) $) -- ($ (h_0) + (-0.47*\statenodesize, -0.15*\statenodesize) + (-5em, 0) $)  -- +(5em, 0) 
        \tikztonodes
        }, swap, m2] node [tl, xshift=-2em] {$a_2, p$} (h_0)

    (h_0) edge [m12] (dots_h_0_h_u)
    (dots_h_0_h_u) edge [m12] (h_u)
    (h_u) edge [m12] (dots_h_u_h_v)
    (dots_h_u_h_v) edge [m12] (h_v)
    (h_v) edge [bend right=20, m12] (h_u)

    (g_l) edge [to path={
        ($ (g_l) + (0.47*\statenodesize, 0.15*\statenodesize) $) -- ($ (q_0) + (-0.47*\statenodesize, 0.15*\statenodesize) $) \tikztonodes
        }, m1] node [tl, xshift=2em] {$a_2, p$} (q_0)

    (g_l) edge [to path={
        ($ (g_l) + (0.47*\statenodesize, -0.15*\statenodesize) $) -- ($ (q_0) + (-0.47*\statenodesize, -0.15*\statenodesize) $) \tikztonodes
        }, swap, m2] node [tl, xshift=2em] {$a_1, p$} (q_0)

    (q_0) edge [m12] (dots_q_0_q_m)
    (dots_q_0_q_m) edge [m12] (q_m)
    (q_m) edge [m12] (dots_q_m_q_n)
    (dots_q_m_q_n) edge [m12] (q_n)
    (q_n) edge [bend right=20, m12] (q_m)

    ;

    \counterexamplemdplegend

\end{tikzpicture}
\caption{Counterexample MDPs $\mathcal{M}_1$ and $\mathcal{M}_2$, with transitions distinguished by arrow types (see legend). 
Both MDPs are parameterized by the parameter $p$ that is in range $0 < p < 1$.
Unlabeled edges are deterministic (actions $a_1$ and $a_2$ transition with probability $1$).
Ellipsis indicates a deterministic chain of states.
}
\label{fig:counterexample-mdps}
\end{figure*}

This section states and outlines the proof to the main result.

By specializing the $\kappa$-PAC definitions  (\Cref{def:rl-obj-pac-algo,def:rl-obj-sample-efficiently-pac-algo}) with the definition of LTL objectives in \Cref{sec:ltl-objectives}, we obtain the following definitions of LTL-PAC.

\begin{definition}
\label{def:rl-ltl-pac-algo}
Given an LTL objective $\xi$, a \planningwithgenerativemodel algorithm  $(\mathcal{A}^\text{S}, \mathcal{A}^\text{L})$ is {\em LTL-PAC} (probably approximated correct for LTL objective $\xi$) in an environment MDP $\mathcal{M}$ for the LTL objective $\xi$ if, with the sequence of transitions $T$ of length $N$ sampled using the sampling algorithm $\mathcal{A}^\text{S}$, the learning algorithm $\mathcal{A}^\text{L}$ outputs a non-Markovian $\epsilon$-optimal policy with a probability of at least $1-\delta$ for all $\epsilon>0$ and $0 < \delta < 1$. That is,
\begin{equation}
\label{eq:pac-ltl-epsilon-delta-inequality}
\prob*{\rv{T} \sim \mdpsamplingproduct{\mathcal{M}}{\mathcal{A}^\text{S}}{N}}{\mdpvaluefunc{\mathcal{A}^\text{L}\left(T\right)}{\mathcal{M}}{\xi} \ge \mdpvaluefunc{\pi^*}{\mathcal{M}}{\xi} - \epsilon} \ge 1 - \delta .
\end{equation}
\end{definition}
We call the probability on the left of the inequality the {\em LTL-PAC probability} of the algorithm $(\mathcal{A}^\text{S}, \mathcal{A}^\text{L})$.

\begin{definition}
\label{def:rl-ltl-sample-efficiently-pac-algo}
Given an LTL objective $\xi$, an LTL\Hyphdash PAC \planningwithgenerativemodel algorithm for $\xi$ is {\em sample} {\em efficiently} LTL\Hyphdash PAC if the number of sampled transitions $N$ is asymptotically polynomial to $\frac{1}{\epsilon}$, $\frac{1}{\delta}$, $|S|$, $|A|$.
\end{definition}

With the above definitions, we can now define the PAC learnability of an LTL objective and state the main theorem.
\begin{definition}
\label{def:rl-ltl-pac-learnable}
An LTL formula $\phi$ over atomic propositions $\Pi$ is {\em LTL-PAC-learnable by \planningwithgenerativemodel (\reinforcementlearning)} if there exists a sample efficiently LTL-PAC \planningwithgenerativemodel (\reinforcementlearning) algorithm for all environment MDPs and all consistent labeling functions $\mathcal{L}$ (that is, $\mathcal{L}$ maps from the MDP's states to $2^\Pi$) for the LTL objective specified by $(\mathcal{L}, \phi)$.
\end{definition}

\begin{theorem}
\label{thm:ltl-non-pac}
An LTL formula $\phi$ is LTL-PAC-learnable by \reinforcementlearning (\planningwithgenerativemodel) if (and only if) $\phi$ is \ltlclasst{finitary}.
\end{theorem}
Between the two directions of \Cref{thm:ltl-non-pac}, the forward direction (``only if'')  is more important.
The forward direction states that for any LTL formula not in $\ltlclass{finitary}$  (that is,  infinite-horizon properties), there does not exist a \planningwithgenerativemodel algorithm---which by definition also excludes any \reinforcementlearning algorithm---that is sample efficiently LTL-PAC for all environments. 
This result is the core contribution of the paper---infinite-horizon LTL formulas are not sample efficiently LTL-PAC-learnable.

Alternatively, the reverse direction of \Cref{thm:ltl-non-pac} states that, for any \ltlclasst{finitary} formula (finite-horizon properties), there exists a \reinforcementlearning algorithm---which by definition is also a \planningwithgenerativemodel algorithm---that is sample efficiently LTL-PAC for all environments.

\subsection{Proof of \Cref{thm:ltl-non-pac}: Forward Direction}
\label{sec:ltlnonpacforwardproof}

This section proves the forward direction of \Cref{thm:ltl-non-pac}.
First, we construct a family of pairs of MDPs.
Then, for the singular case of the LTL formula $\eventually h_0$, we derive a sample complexity lower bound for any LTL-PAC \planningwithgenerativemodel algorithm applied to our family of MDPs. 
This lower bound necessarily depends on a specific transition probability in the MDPs.
Finally, we generalize this bound to any non-finitary LTL formula and conclude the proof.

\subsubsection{MDP Family}
\label{sec:counterexample-mdps}

We give two constructions of parameterized counterexample MDPs $\mathcal{M}_1$ and $\mathcal{M}_2$ shown in \Cref{fig:counterexample-mdps}.
The key design behind each pair in the family is that no \planningwithgenerativemodel algorithm can learn a policy that is simultaneously $\epsilon$-optimal on both MDPs without observing a number of samples that depends on the probability of a specific transition. 

Both MDPs are parameterized by the shape parameters $k$, $l$, $u$, $v$, $m$, $n$, and an unknown transition probability parameter $p$.
The actions are $\{a_1, a_2\}$, and the state space is partitioned into three regions (as shown in \Cref{fig:counterexample-mdps}: states $g_{0 \dots l}$ (the grey states), states $h_{0\dots v}$ (the line-hatched states), and states $q_{0\dots n}$ (the white states).
All transitions, except $g_l \rightarrow h_0$ and $g_l \rightarrow q_0$, are the same between $\mathcal{M}_1$ and $\mathcal{M}_2$.
The effect of this difference between the two MDPs is that, for $\mathcal{M}_i$, $i \in \{1, 2\}$:
\begin{itemize}[leftmargin=*, wide=0pt]
\item Action $a_i$ in $\mathcal{M}_i$ at the state $g_l$ will transition to the state $h_0$ with probability $p$, inducing a run that cycles in the region $h_{u \dots v}$ forever.
\item Action $a_{3-i}$ (the alternative to $a_i$) in $\mathcal{M}_i$ at the state $g_l$ will transition to the state $q_0$ with probability $p$, inducing a run that cycles in the region $q_{m \dots n}$ forever.
\end{itemize}

Further, for any policy, a run of the policy on both MDPs must eventually reach $h_0$ or $q_0$ with probability $1$, and ends in an infinite cycle in either $h_{u\dots  v}$ or $q_{m\dots n}$.

\subsubsection{Sample Complexity of \texorpdfstring{$\eventually h_0$}{F h0}}
\label{sec:samplecomplexityFh0}

We next consider the LTL objective $\xi^{h_0}$ specified by the LTL formula $\eventually h_0$ and the labeling function $\mathcal{L}^{h_0}$ that labels only the state $h_0$ as $\mathit{true}$.
A sample path on the MDPs (\Cref{fig:counterexample-mdps}) satisfies this objective iff the path reaches the state $h_0$.

Given \begingroup\small$\epsilon > 0$ \endgroup  and \begingroup\small$0 < \delta < 1$\endgroup, our goal is to derive a lower bound on the number of sampled environment transitions performed by an algorithm, so that the satisfaction probability of $\pi$, the learned policy, is $\epsilon$-optimal (i.e., \begingroup\small$\mdpvaluefunc{\pi}{\mathcal{M}}{\xi^{h_0}} \ge \mdpvaluefunc{\pi^*}{\mathcal{M}}{\xi^{h_0}} - \epsilon$\endgroup) with a probability of least $1 - \delta$.

The key rationale behind the following lemma is that, if a \planningwithgenerativemodel algorithm has not observed any transition to either $h_0$ or $q_0$, the learned policy cannot be $\epsilon$-optimal in both $\mathcal{M}_1$ and $\mathcal{M}_2$.

\begin{lemma}
\label{thm:minofm1m2pacprob}
For any \planningwithgenerativemodel algorithm $(\mathcal{A}^\text{S}, \mathcal{A}^\text{L})$, it must be the case that:
$
\min\left(\zeta_1, \zeta_2\right) \le \frac{1}{2}
$,
where {\small $ \zeta_i = \prob*{\rv{T}}{
\mdpvaluefunc{\mathcal{A}^\text{L}(T)}{\mathcal{M}_i}{\xi^{h_0}} \ge \mdpvaluefunc{\pi^*}{\mathcal{M}_i}{\xi^{h_0}} - \epsilon | n\left(T\right) = 0}$} and $n(T)$ is the number of transitions in $T$ that start from $g_l$ and end in either $h_0$ or $q_0$. 
\end{lemma}

The value $\zeta_i$ is the LTL-PAC probability of a learned policy on $\mathcal{M}_i$, given that the \planningwithgenerativemodel algorithm did not observe any information that allows the algorithm to distinguish between $\mathcal{M}_1$ and $\mathcal{M}_2$. 

\newlist{steps}{enumerate}{1}
\setlist[steps, 1]{label=\textbf{Step \arabic*}, align=left, nosep}%

\begin{proof}
We present a proof of \Cref{thm:minofm1m2pacprob} in \Cref{sec:proofofminofm1m2pacprob}.
\end{proof}

A \planningwithgenerativemodel algorithm cannot learn an $\epsilon$-optimal policy without observing a transition to either $h_0$ or $q_0$. 
Therefore, we bound the sample complexity of the algorithm from below by the probability that the sampling algorithm does observe such a transition:

\begin{lemma}
\label{thm:ltlpaclowerboundFh0}
For the LTL objective $\xi^{h_0}$, the number of samples, $N$, for an LTL-PAC  \planningwithgenerativemodel algorithm for both $\mathcal{M}_1$ and $\mathcal{M}_2$ (for any instantiation of the parameters $k,l,u,v,m,n$) has a lower bound of
$
N \ge \frac{ \log (2\delta) }{\log \left(1-p\right)}
$.
\end{lemma}
Below we give a proof sketch of \Cref{thm:ltlpaclowerboundFh0}; we give the complete proof in \Cref{sec:ltlpaclowerboundFh0fullproof}.
\begin{proof}[Proof Sketch of \Cref{thm:ltlpaclowerboundFh0}]
First, we assert that the two inequalities of \Cref{eq:pac-ltl-epsilon-delta-inequality} for both $\mathcal{M}_1$ and $\mathcal{M}_2$ holds true for a \planningwithgenerativemodel algorithm.
Next, by conditioning on $n(T) = 0$, plugging in the notation of $\zeta_i$, and relaxing both inequalities, we get 
$
(1 - \zeta_i) \prob{\rv{T}}{n(T) = 0} \le \delta
$, for $i \in \{1, 2\}$.
Then, since $n(T) = 0$ only occurs when all transitions from $g_l$ end in $g_k$, we have $\prob{\rv{T}}{n(T) = 0}  \ge (1 - p)^N$.
Combining the inequalities, we get $(1 - \min(\zeta_1, \zeta_2))(1 - p)^N \le \delta$.
Finally, we apply \Cref{thm:minofm1m2pacprob} to get the desired lower bound of $N\,{\ge}\,\frac{\log(2\delta)}{\log(1 - p)}$.
\end{proof}

\subsubsection{Sample Complexity of \texorpdfstring{Non-\ltlclasst{finitary}}{Non-finitary} Formulas}
\label{sec:samplecomplexitynonfinitary}

This section generalizes our lower bound on $\eventually h_0$ to all non-\ltlclasst{finitary} LTL formulas.
The key observation is that for any non-\ltlclasst{finitary} LTL formula, we can choose a pair of MDPs, $\mathcal{M}_1$ and $\mathcal{M}_2$, from our MDP family.
For both MDPs in this pair, finding an $\epsilon$-optimal policy for $\eventually h_0$ is reducible to finding an $\epsilon$-optimal policy for the given formula. 
By this reduction, the established lower bound for the case of $\eventually h_0$ also applies to the case of any non-\ltlclasst{finitary} formula.
Therefore, the sample complexity of learning an $\epsilon$-optimal policy for any non-\ltlclasst{finitary} formula has a lower bound of $\frac{\log(2\delta)}{\log \left(1-p\right)}$.

We will use $[w_1; w_2 ; \dots w_n]$ to denote the concatenation of the finite-length words $w_1 \dots w_n$.
We will use $\wrepeat{w}{i}$ to denote the repetition of the finite-length word $w$ by $i$ times, and $\wcycle{w}$ to denote the infinite repetition of $w$.

\begin{definition}
\label{def:uncommittable-word}
An accepting (resp.\ rejecting) infinite-length word $\wrational{w_a}{w_b}$ of $\phi$ is \begingroup\em uncommittable \endgroup if there exists finite-length words $w_c$, $w_d$ such that $\phi$ rejects (resp.\ accepts) $\wrational{w_a; \wrepeat{w_b}{i}; w_c}{w_d}$ for all $i \in \naturals$.
\end{definition}
\begin{lemma}
\label{thm:uncommitablewordssamesatprob}
If $\phi$ has an uncommittable word $w$, there is an instantiation of $\mathcal{M}_1$ (or $\mathcal{M}_2$) in \Cref{fig:counterexample-mdps} and a labeling function $\mathcal{L}$, such that, for any policy, the satisfaction probabilities of that policy in $\mathcal{M}_1$ (or $\mathcal{M}_2$) for the LTL objectives specified by $(\mathcal{L}, \phi)$ and $(\mathcal{L}^{h_0}, \eventually h_0)$ are the same.
\end{lemma}
\begin{proof}
For an uncommittable word $w$, we first find the finite-length words $w_a$,$w_b$,$w_c$,$w_d$ according to \Cref{def:uncommittable-word}.
We then instantiate $\mathcal{M}_1$ and $\mathcal{M}_2$ in \Cref{fig:counterexample-mdps} as follows.
\begin{itemize}[noitemsep, topsep=0pt, partopsep=0pt, leftmargin=*, wide=0pt]
\item If $w$ is an uncommittable accepting word, we set $k$, $l$, $u$, $v$, $m$, $n$ (\Cref{fig:counterexample-mdps}) to $|w_a|$, $|w_a| + |w_b|$, $0$, $|w_b|$, $|w_c|$ and $|w_c|+|w_d|$, respectively. We then set the labeling function as in \Cref{eq:guaranteemdplabelfunc}.
\item If $w$ is an uncommittable rejecting word, we set $k$, $l$, $u$, $v$, $m$, $n$ (\Cref{fig:counterexample-mdps}) to $|w_a|$, $|w_a| + |w_b|$, $|w_c|$, $|w_c| + |w_d|$, $0$ and $|w_b|$, respectively. We then set the labeling function as in \Cref{eq:safetymdplabelfunc}.
\end{itemize}
\vspace{-0.8em}
\[
\arraycolsep=1.3pt\def\arraystretch{1.3}
\begin{array}{c@{\phantom{x}}c}
\mathcal{L}(s){=}\left\{
\arraycolsep=1.2pt\def\arraystretch{1.1}
\begin{array}{ll}
\windex{[w_a; w_b]}{j} & \text{if } s{=}g_j \\
\windex{w_b}{j} & \text{if } s{=}h_j \\
\windex{[w_c; w_d]}{j} & \text{if } s{=}q_j \\
\end{array}\right. & %
\mathcal{L}(s){=}\left\{%
\arraycolsep=1.4pt\def\arraystretch{1.2}
\begin{array}{ll}
\windex{[w_a; w_b]}{j} & \text{if } s {=}g_j \\
\windex{[w_c; w_d]}{j} & \text{if } s{=}h_j \\
\windex{w_b}{j} & \text{if } s{=}q_j \\
\end{array}\right. \\
\addlabel{eq:guaranteemdplabelfunc} & \addlabel{eq:safetymdplabelfunc}
\end{array}
\]
\vspace{-1em}

In words, for an uncommittable accepting word, we label the states $g_{0\dots l}$ one-by-one by $[w_a; w_b]$;
we label the states $h_{0 \dots v}$ one-by-one by $w_b$ (and set $u = 0$, which eliminates the chain of states $h_{0\dots u}$);
we label the states $q_{0 \dots n}$ one-by-one by $[w_c; w_d]$.
Symmetrically, for an uncommittable rejecting word, we label the states $g_{0\dots l}$ one-by-one by $[w_a; w_b]$;
we label the states $h_{0 \dots v}$ one-by-one by $[w_c; w_d]$;
we label the states $q_{0 \dots n}$ one-by-one by $w_b$ (and set $m = 0$, which eliminates the chain of states $q_{0\dots m}$).

By the above instantiation, the two objectives specified by $(\mathcal{L}, \phi)$ and $(\mathcal{L}^{h_0}, \eventually h_0)$ are equivalent in $\mathcal{M}_1$ and $\mathcal{M}_2$.  
In particular, any path in $\mathcal{M}_1$ or $\mathcal{M}_2$ satisfies the LTL objective specified by $(\mathcal{L}, \phi)$ if and only if the path visits the state $h_0$ and therefore also satisfies the LTL objective specified by $(\mathcal{L}^{h_0}, \eventually h_0)$.
Therefore, any policy must have the same satisfaction probability for both objectives.
\end{proof}

\begin{lemma}
\label{thm:ltlpaclowerbound}
For $\phi\notinltlclass{finitary}$, the number of samples for a \planningwithgenerativemodel algorithm to be LTL-PAC has a lower bound of
$
N\ge\frac{\log(2\delta)}{\log(1 - p)}$.
\end{lemma}
\vspace{-1em}
\begin{proof}
A corollary of \Cref{thm:uncommitablewordssamesatprob} is: 
for any $\phi$ that has an uncommittable word, we can construct a pair of MDPs $\mathcal{M}_1$ and $\mathcal{M}_2$ in the family of pairs of MDPs in \Cref{fig:counterexample-mdps}, such that, in both MDPs,
a policy is sample efficiently LTL-PAC for the LTL objective specified by $(\mathcal{L}, \phi)$ if it is sample efficiently LTL-PAC for the LTL objective specified by $(\mathcal{L}^{h_0}, \eventually h_0)$.
This property implies that the lower bound in \Cref{thm:ltlpaclowerboundFh0} for the objective specified by $(\mathcal{L}^{h_0}, \eventually h_0)$ also applies to the objective specified by $(\mathcal{L}, \phi)$, provided that any $\phi\notinltlclass{finitary}$ has an uncommittable word.
In \Cref{sec:proofofltlnonfinitaryprop}, we prove a lemma that any formula $\phi\notinltlclass{guarantee}$ has an uncommittable accepting word, and any formula $\phi\notinltlclass{safety}$ has an uncommittable rejecting word.
Since \ltlclass{finitary} is the intersection of \ltlclass{guarantee} and \ltlclass{safety},
this completes the proof.
\end{proof}

\subsubsection{Conclusion}

Note that the lower bound $\frac{\log(2\delta)}{\log(1 - p)}$ depends on $p$, the transition probability in the constructed MDPs.
Moreover, for \begingroup\small$\delta < \frac{1}{2}$\endgroup, as $p$ approaches $0$, this lower bound goes to infinity.
As a result, the bound does not satisfy the definition of sample efficiently LTL-PAC \planningwithgenerativemodel algorithm for the LTL objective (\Cref{def:rl-obj-sample-efficiently-pac-algo}), and thus no algorithm is sample efficiently LTL-PAC. Therefore, LTL formulas not in \ltlclass{finitary} are not LTL-PAC-learnable.
This completes the proof of the forward direction of \Cref{thm:ltl-non-pac}.

\subsection{Proof Sketch of \Cref{thm:ltl-non-pac}:  Reverse Direction}
\label{sec:ltlnonpacreverseproofoutline}

This section gives a proof sketch to the reverse direction of \Cref{thm:ltl-non-pac}. 
We give a complete proof in \Cref{sec:ltlnonpacreverseproof}.

We prove the reverse direction of \Cref{thm:ltl-non-pac} by reducing the problem of learning a policy for any \ltlclasst{finitary} formula to the problem of learning a policy for a finite-horizon cumulative rewards objective. 
We conclude the reverse direction of the theorem by invoking a known PAC \reinforcementlearning algorithm on the later problem.
\begingroup
\newcommand{\litem}[1]{\item {\bfseries #1.}}
\begin{itemize}[leftmargin=*, wide=0pt]
\litem{Reduction to Infinite-horizon Cumulative Rewards}
First, given an LTL formula in \ltlclass{finitary} and an environment MDP, we will construct an {\em augmented MDP with rewards} similar to \citet{ltlf-rl,rewardmachine1}.
We reduce the problem of finding the optimal non-Markovian policy for satisfying the formula in the original MDP to the problem of finding the optimal Markovian policy that maximizes the infinite-horizon (undiscounted) cumulative rewards in this augmented MDP.
\litem{Reduction to Finite-horizon Cumulative Rewards} Next, we reduce the infinite-horizon cumulative rewards to a finite-horizon cumulative rewards, using the fact that the formula is \ltlclasst{finitary}.
\litem{Sample Complexity Upper Bound} Lastly, \citet{ipoc19} have derived an upper bound on the sample complexity for a \reinforcementlearning algorithm for finite-horizon MDPs. 
We thus specialize this known upper bound to our problem setup of the augmented MDP and conclude that any \ltlclasst{finitary} formula is PAC-learnable.
\end{itemize}
\endgroup

\subsection{Consequence of the Core Theorem}

\Cref{thm:ltl-non-pac} implies that: 
For any non-\ltlclasst{finitary} LTL objective, given any arbitrarily large finite sample of transitions, the learned policy need not perform near-optimally.
This implication is unacceptable in applications that require strong guarantees of the overall system's behavior.

\section{Empirical Justifications}
\label{sec:empirical}

\DeclareDocumentCommand{\rewardscheme}{m}{%
\IfEqCase{#1}{%
    {multi-discount}{Multi-discount}%
    {zeta-reach}{Zeta-reach}%
    {zeta-acc}{Zeta-acc}%
    {zeta-discount}{Zeta-discount}%
    {reward-on-acc}{Reward-on-acc}%
}[\PackageError{rewardscheme}{Undefined option to \rewardscheme: #1}{}]%
}

\DeclareDocumentCommand{\rlalgo}{m}{%
\IfEqCase{#1}{%
    {Q}{Q-learning}%
    {DQ}{Double Q-learning}%
    {SL}{SARSA$\left(\lambda\right)$}%
}[\PackageError{rlalgo}{Undefined option to \rlalgo: #1}{}]%
}

This section empirically demonstrates our main result, the forward direction of \Cref{thm:ltl-non-pac}.

Previous work has introduced various \reinforcementlearning algorithms for LTL objectives \citep{dorsa,omegaregularrl19,hasanbeig2019reinforcement,bozkurt2020control}.
We ask the research question:
\begingroup\em Do the sample complexities of these algorithms depend on the transition probabilities of the environment? \endgroup
To answer the question, we 
evaluate various algorithms and empirically measure the sample sizes for them to obtain near-optimal policies with high probability.

\subsection{Methodology}

We consider various recent \reinforcementlearning algorithms for LTL objectives \citep{dorsa,omegaregularrl19,bozkurt2020control}.
We consider two pairs of LTL formulas and environment MDPs (LTL-MDP pair).
The first pair is the formula $\eventually h$ and the counterexample MDP as shown in \Cref{fig:counterexample-mdp-simple}.
The second pair is adapted from a case study in \citet{dorsa}.
We focus on the first pair in this section and defer the complete evaluation to \Cref{sec:more_empirical}.

We run the considered algorithms on each chosen LTL-MDP pair with a range of values for the parameter $p$ and let the algorithms perform $N$ environment samples.
For each algorithm and each pair of values of $p$ and $N$, we fix $\epsilon = 0.1$ and repeatedly run the algorithm to obtain a Monte Carlo estimation of the LTL-PAC probability (left side of \Cref{eq:pac-ltl-epsilon-delta-inequality}) for that setting of $p$, $N$ and $\epsilon$.
We repeat each setting until the estimated standard deviation of the estimated probability is within $0.01$. 
In the end, for each algorithm and LTL-MDP pair we obtain $5 \times 21 = 105$ LTL-PAC probabilities and their estimated standard deviations.

For the first LTL-MDP pair, we vary $p$ by a geometric progression from \begingroup\small$10^{-1}$ \endgroup  to \begingroup\small$10^{-3}$ \endgroup in $5$ steps.
We vary \small$N$ \normalsize by a geometric progression from \small$10^1$ \normalsize to \small$10^5$ \normalsize in \small$21$ \normalsize steps.
For the second LTL-MDP pair, we vary $p$ by a geometric progression from $0.9$ to $0.6$ in $5$ steps.
We vary $N$ by a geometric progression from $3540$ to $9\times10^4$ in $21$ steps. 
If an algorithm does not converge to the desired LTL-PAC probability within $9\times10^4$ steps, we rerun the experiment with an extended range of $N$ from $3540$ to $1.5\times10^5$. 

\begin{figure}[t]
\centering
\includegraphics[width=0.55\linewidth]{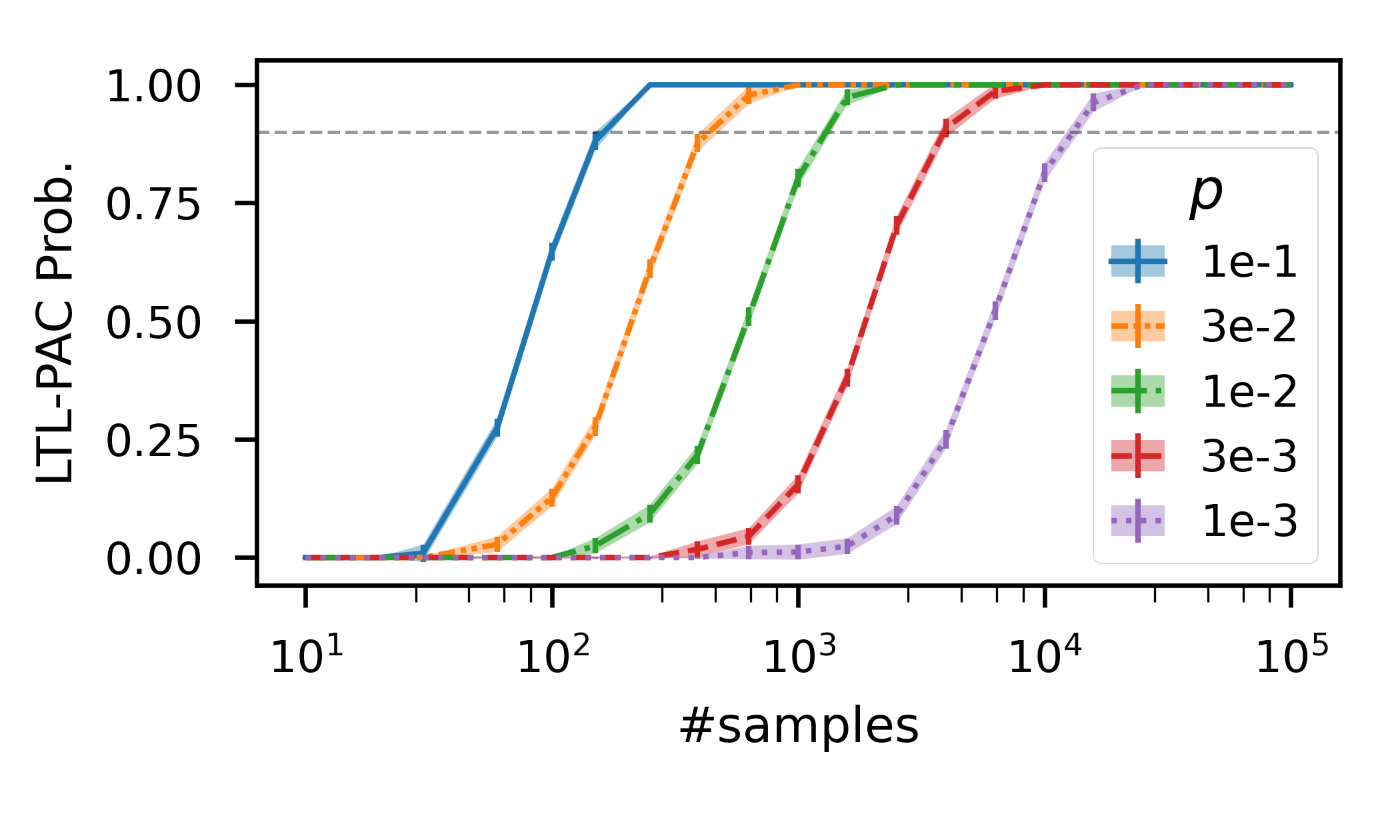}
\includegraphics[width=0.35\linewidth]{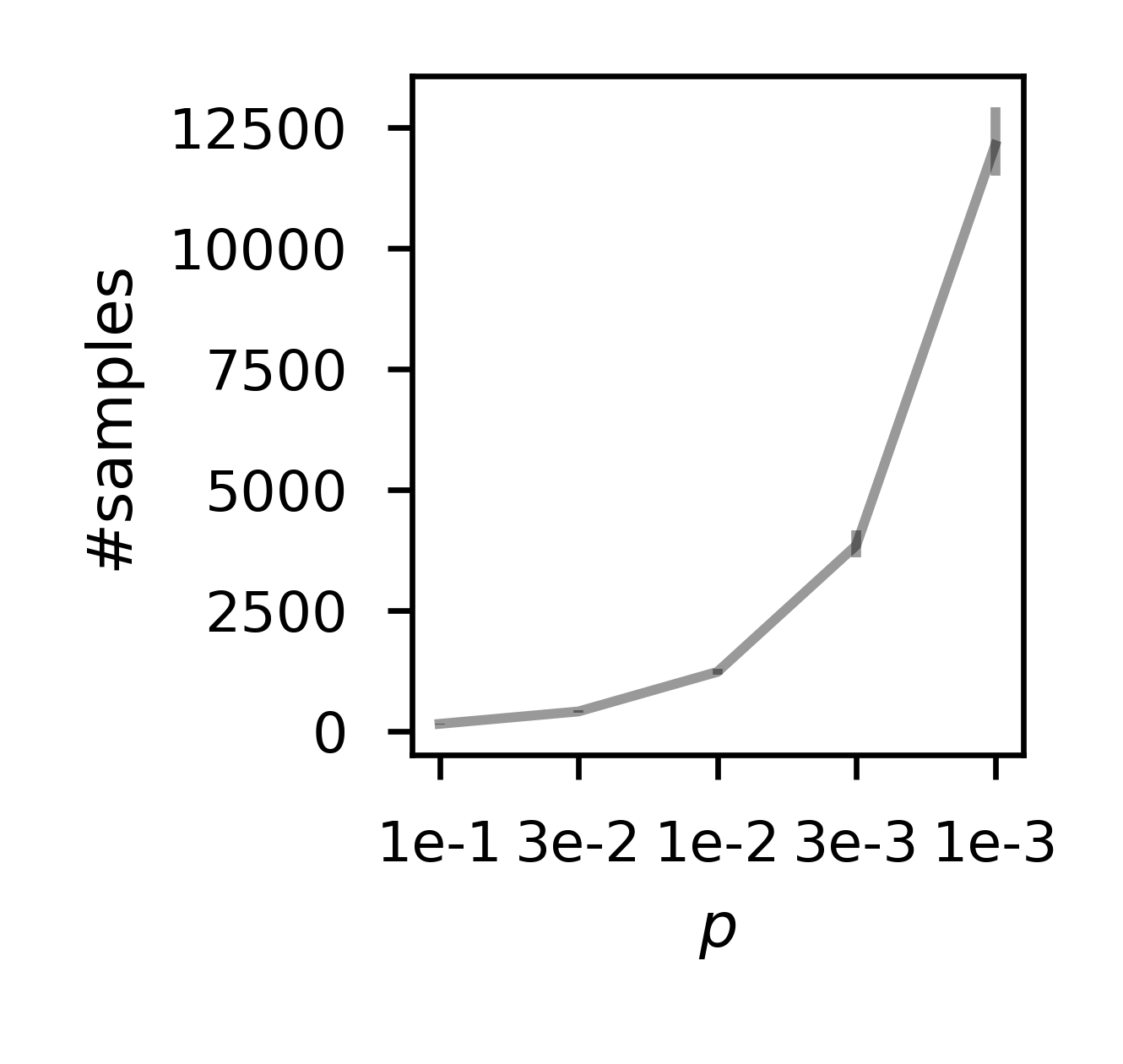}
\caption{Left: LTL-PAC probabilities vs.\ number of samples, varying parameters $p$. Right: number of samples needed to reach 0.9 LTL-PAC probability vs.\ parameter $p$.}
\label{fig:bozkurt_pac_probs_plot}
\end{figure}

\subsection{Results}

\Cref{fig:bozkurt_pac_probs_plot} presents the results for the algorithm in \citet{bozkurt2020control} with the setting of \rewardscheme{multi-discount}, \rlalgo{Q}, and the first LTL-MDP pair.
On the left, we plot the LTL-PAC probabilities vs.\ the number of samples $N$, one curve for each $p$.
On the right, we plot the intersections of the curves in the left plot with a horizontal cutoff of $0.9$.

As we see from the left plot of \Cref{fig:bozkurt_pac_probs_plot}, for each $p$, the curve starts at $0$ and grows to $1$ in a sigmoidal shape as the number of samples increases.
However, as $p$ decreases, the MDP becomes harder: As shown on the right plot of \Cref{fig:bozkurt_pac_probs_plot}, the number of samples required to reach the particular LTL-PAC probability of $0.9$ grows exponentially.
Results for other algorithms, environments and LTL formulas are similar and lead to the same conclusion.

\subsection{Conclusion} 
Since the transition probabilities ($p$ in this case) are unknown in practice, one can't know which curve in the left plot a given environment will follow.
Therefore, given any finite number of samples, these \reinforcementlearning algorithms cannot provide guarantees on the LTL-PAC probability of the learned policy.
This result supports \Cref{thm:ltl-non-pac}.

\section{Directions Forward}
\label{sec:directions_forward}

We have established the intractability of reinforcement learning for infinite\Hyphdash horizon LTL objectives. 
Specifically, for any infinite\Hyphdash horizon LTL objective, the learned policy need not perform near-optimally given any finite number of environment interactions.
This intractability is undesirable in applications that require strong guarantees, such as traffic control, robotics, and autonomous vehicles~\citep{collisionavoidance2010,rlroboticssurvey,safeautonomousvehciles18}.

Going forward, we categorize approaches that either focus on tractable objectives or weaken the guarantees required by an LTL-PAC algorithm.
We obtain the first category from the reverse direction of \Cref{thm:ltl-non-pac}, and each of the other categories by relaxing a specific requirement that \Cref{thm:ltl-non-pac} places on an algorithm.
Further, we classify previous approaches into these categories.

\subsection{Use a \texorpdfstring{\ltlclasst{Finitary}}{Finitary} Objective}

Researchers have introduced specification languages that express \ltlclasst{finitary} properties and have applied reinforcement learning to objectives expressed in these languages~\citep{smcbltl,osbert}.
One value proposition of these approaches is that they provide succinct specifications because \ltlclasst{finitary} properties written in LTL directly are verbose.
For example, the \ltlclasst{finitary} property ``$a$ holds for $100$ steps'' is equivalent to an LTL formula with a conjunction of $100$ terms:
\begingroup\small$
a \land \ltlnext a \land \dots \land (\underbrace{\ltlnext\dots\ltlnext}_{99 \text{ times}} a).  
$\endgroup

For these succinct specification languages, by the reduction of these languages to \ltlclasst{finitary} properties and the reverse direction of \Cref{thm:ltl-non-pac}, there exist reinforcement-learning algorithms that give LTL-PAC guarantees.

\subsection{Best-effort Guarantee}
The definition of LTL-PAC (\Cref{def:rl-ltl-pac-algo}) requires a \reinforcementlearning algorithm to learn a policy with satisfaction probability within $\epsilon$ of optimal, for all $\epsilon\,{>}\,0$.
However, it is possible to relax this quantification over $\epsilon$ so that an algorithm only returns a policy with the best-available $\epsilon$ it finds.

For example, \citet{pacmdpsmcconfidenceinterval} introduced a \reinforcementlearning algorithm for objectives in the \ltlclass{guarantee} class.
Using a specified time budget, the algorithm returns a policy and an $\epsilon$. 
Notably, it is possible for the returned $\epsilon$ to be $1$,
a vacuous bound on performance.

\subsection{Know More About the Environment}
The definition of LTL\Hyphdash PAC  (\Cref{def:rl-ltl-pac-algo}) requires a \reinforcementlearning algorithm to provide a guarantee for all environments.
However, on occasion, one can have prior information on the transition probabilities of the MDP at hand. 

For example, \citet{ufukpacmdpltl} introduced a \reinforcementlearning algorithm with a PAC-MDP guarantee that depends on the time horizon until the MDP reaches a steady state.
Given an MDP, this time horizon is generally unknown; however, if one has knowledge of this time horizon \emph{a priori}, it constrains the set of MDPs and yields an LTL-PAC guarantee dependent on this time horizon.

As another example, \citet{smcpmin} introduced a \reinforcementlearning algorithm that provides an LTL-PAC guarantee provided a declaration of the minimum transition probability of the MDP. This constraint, again, bounds the space of considered MDPs.

\subsection{Use an LTL-like Objective}

\Cref{thm:ltl-non-pac} only considers LTL objectives.
However, one opportunity for obtaining a PAC guarantee is to change the problem---use a specification language that is LTL-like, defining similar temporal operators but also giving those operators a different, less demanding, semantics.

\newcommand{\hyperparam}{\ensuremath{\boldmath{\lambda}}}

\subsubsection{LTL-in-the-limit Objectives}

One line of work \citep{dorsa,omegaregularrl19,hasanbeig2019reinforcement,bozkurt2020control} uses LTL formulas as the objective, but also introduces one or more hyper-parameters $\hyperparam$ to relax the formula's semantics.
The \reinforcementlearning algorithms in these works learn a policy for the environment MDP given fixed values of the hyper-parameters.
Moreover, as hyper-parameter values approach a limit point, the learned policy becomes optimal for the hyper-parameter-free LTL formula.\footnote{\citet{omegaregularrl19} and \citet{bozkurt2020control} showed that there exists a critical setting of the parameters $\lambda^*$ that produces the optimal policy. However, $\lambda^*$ depends on the transition probabilities of the MDP and is therefore consistent with our findings.}
The relationship between these relaxed semantics and the original LTL semantics is analogous to the relationship between discounted and average-reward infinite-horizon MDPs. 
Since discounted MDPs are PAC-MDP-learnable \citep{pacmodelfreerl}, we conjecture that these relaxed LTL objectives (at any fixed hyper-parameter setting) are PAC-learnable.

\subsubsection{General LTL-like Objectives}

Prior approaches \citep{gltl17,truncatedltl-rl,ltlf-rl,rewardmachine1} also use general LTL-like specifications that do not or are not known to converge to LTL in a limit.
For example, \citet{rewardmachine1} introduced the reward-machine objective that uses a finite state automaton to specify a reward function.
As another example, \citet{gltl17} introduced {\em geometric LTL}. Geometric LTL attaches a geometrically distributed horizon to each temporal operator. 
The learnability of these general LTL-like objectives is a potential future research direction.

\section{Conclusion}
\label{sec:conclusion}

In this work, we have formally proved that infinite-horizon LTL objectives in reinforcement learning cannot be learned in unrestricted environments.
By inspecting the core result, we have identified various possible directions forward for future research.
Our work resolves the apparent lack of a formal treatment of this fundamental limitation of infinite-horizon objectives, helps increase the community's awareness of this problem, and will help organize the community's efforts in reinforcement learning with LTL objectives. 

\bibliographystyle{named}
\newcommand{\arxiv}[1]{arXiv preprint: #1}

\ifseparateappendix
\else
\clearpage
\begin{appendices}
\crefalias{section}{appendix}
\crefalias{subsection}{appendix}
\section{Background on Linear Temporal Logic}
\label{sec:ltl-background-full}

\subsection{LTL Semantics}

In this section, we give the formal semantics of LTL. Recall that the satisfaction relation $w \vDash \phi$ denotes that the infinite word $w$ satisfies $\phi$.
\Cref{eq:ltl_semantics} defines this relation.
\begin{equation}
\label{eq:ltl_semantics}
\newcommand{\txtiff}{&\text{ iff }&}
\begin{aligned}
w & \vDash a  \txtiff a \in \windex{w}{0} \quad a \in \Pi \\
w & \vDash \neg \phi \txtiff w \nvDash \phi \\
w & \vDash \phi \land \psi \txtiff w \vDash \phi \land w \vDash \psi \\
w & \vDash \ltlnext \phi \txtiff \wsuffix{w}{1} \vDash \phi \\
w & \vDash \phi \until \psi \txtiff \exists j \ge 0 \ldotp \big( \wsuffix{w}{j} \vDash \psi \,\,\land  \\ & & & \;\; \forall k \ge 0 \ldotp k < j \implies \wsuffix{w}{k} \vDash \phi \big).
\end{aligned}
\end{equation}
The rest of the operators are defined as syntactic sugar in terms of operators in \Cref{eq:ltl_semantics} as: $\phi \lor \psi \equiv \neg (\neg \phi \land \neg \psi)$, $\eventually \phi \equiv \text{True} \until \phi$, $\always \phi \equiv \neg \eventually \neg \phi$.

We describe the semantics of each operator in words below. 
\begin{itemize}
\item $\ltlnext \phi$: the sub-formula $\phi$ is true in the next time step.
\item $\always \phi$: the sub-formula $\phi$ is always true in all future time steps.
\item $\eventually \phi$: the sub-formula $\phi$ is eventually true in some future time steps.
\item $\phi \until \psi$: the sub-formula $\phi$ is always true until the sub-formula $\psi$ eventually becomes true, after which $\phi$ is allowed to become false.
\end{itemize}

\subsection{Complete Description of the LTL Hierarchy}
\label{sec:ltl-hierarchy-full}

In this section, we describe the key properties of all classes in the LTL hierarchy (see \Cref{fig:ltl-hierarchy}).

\begin{itemize}[leftmargin=*]
\item $\phi\inltlclass{finitary}$ iff there exists a horizon $H$ such that infinite-length words sharing the same prefix of length $H$ are either all accepted or all rejected by $\phi$. E.g., $a \land \ltlnext a$ (i.e., $a$ is true for two steps) is in \ltlclass{finitary}.
\item $\phi\inltlclass{guarantee}$ iff there exists a language of finite words $L$ (i.e., a Boolean function on finite-length words) such that $w \vDash \phi$ if $L$ accepts a prefix of $w$. Informally, a formula in \ltlclass{guarantee} asserts that something eventually happens. E.g., $\eventually a$ (i.e., eventually $a$ is true) is in \ltlclass{guarantee}.
\item $\phi\inltlclass{safety}$ iff there exists a language of finite words $L$ such that $w \vDash \phi$ if $L$ accepts all prefixes of $w$. Informally, a formula in \ltlclass{safety} asserts that something always happens. E.g., $\always a$ (i.e., $a$ is always true) is in \ltlclass{safety}.
\item $\phi \inltlclass{obligation}$ iff $\phi$ is a logical combination of formulas in \ltlclass{guarantee} and \ltlclass{safety}. E.g., $\eventually a \land \always b$ is in \ltlclass{obligation}.
\item $\phi \inltlclass{persistence}$ iff there exists a language of finite words $L$ such that $w \vDash \phi$ if $L$ accepts all but finitely many prefixes of $w$. Informally, a formula in \ltlclass{persistence} asserts that something happens finitely often. E.g., $\eventually\always a$ (i.e., $a$ is not true for only finitely many times, and eventually $a$ stays true forever) is in \ltlclass{persistence}.
\item $\phi \inltlclass{recurrence}$ iff there exists a language of finite words $L$ such that $w \vDash \phi$ if $L$ accepts infinitely many prefixes of $w$. Informally, a formula in \ltlclass{recurrence} asserts that something happens infinitely often. E.g., $\always\eventually a$ (i.e., $a$ is true for infinitely many times) is in \ltlclass{recurrence}.
\item $\phi \inltlclass{reactivity}$ iff $\phi$ is a logical combination of formulas in \ltlclass{recurrence} and \ltlclass{persistence}. E.g., $\always\eventually a \land \eventually\always b$ is in \ltlclass{reactivity}.
\end{itemize}

\section{Proof of \Cref{thm:minofm1m2pacprob}}
\label{sec:proofofminofm1m2pacprob}

To the end of proving \Cref{thm:minofm1m2pacprob}, we first observe the following proposition:
\begin{proposition}
\label{thm:m1m2valuessumtoone}
For any non-Markovian policy $\pi$, the satisfaction probabilities for $\mathcal{M}_1$ and $\mathcal{M}_2$ sum to one: 
\begin{equation*}
\mdpvaluefunc{\pi}{\mathcal{M}_1}{\xi^{h_0}} + \mdpvaluefunc{\pi}{\mathcal{M}_2}{\xi^{h_0}} = 1.
\end{equation*}
\end{proposition}
We give a proof of \Cref{thm:m1m2valuessumtoone} in \Cref{sec:proof-m1m2valuessumtoone}.

\begin{proof}[Proof of \Cref{thm:minofm1m2pacprob}]
Note that the optimal satisfaction probabilities in both $\mathcal{M}_1$ and $\mathcal{M}_2$ is one, that is,  $\mdpvaluefunc{\pi^*}{\mathcal{M}_i}{\xi^{h_0}} = 1$. This is because the policy that always chooses $a_i$ in $\mathcal{M}_i$ guarantees visitation to the state $h_0$.
Therefore, a corollary of \Cref{thm:m1m2valuessumtoone} is that for any policy $\pi$ and any $\epsilon < \frac{1}{2}$, the policy $\pi$ can only be $\epsilon$-optimal in one of $\mathcal{M}_1$ and $\mathcal{M}_2$. Specifically, we have:
\begin{equation}
\small
\label{eq:m1m2valuessumtoonecorollary}
\indicator{\left(\mdpvaluefunc{\pi}{\mathcal{M}_1}{\xi^{h_0}} \ge \mdpvaluefunc{\pi^*}{\mathcal{M}_1}{\xi^{h_0}} - \epsilon\right)} + 
\indicator{\left(\mdpvaluefunc{\pi}{\mathcal{M}_2}{\xi^{h_0}} \ge \mdpvaluefunc{\pi^*}{\mathcal{M}_2}{\xi^{h_0}} - \epsilon\right)} \le 1 .
\end{equation}

Consider a specific sequence of transitions $T$ of length $N$ sampled from either $\mathcal{M}_1$ or $\mathcal{M}_2$.
If $n(T) = 0$, the probability of observing $T$ in $\mathcal{M}_1$ equals to the probability of observing $T$ in $\mathcal{M}_2$, that is:
\begin{equation*}
\begin{split}
& \prob{\rv{T}\sim\mdpsamplingproduct{\mathcal{M}_1}{\mathcal{A}^\text{S}}{N}}{\rv{T} = T | n(T) = 0} \\
& = \prob{\rv{T}\sim\mdpsamplingproduct{\mathcal{M}_2}{\mathcal{A}^\text{S}}{N}}{\rv{T} = T | n(T) = 0} .
\end{split}
\end{equation*}
This is because the only differences between $\mathcal{M}_1$ and $\mathcal{M}_2$ are the transitions $g_l \rightarrow h_0$ and $g_l \rightarrow q_0$, and
conditioning on $n(T) = 0$ effectively eliminates these differences.

Therefore, we can write the sum of $\zeta_1$ and $\zeta_2$ as:
\begin{multline*}
\zeta_1 + \zeta_2 = 
\sum_{\forall T} \prob{\rv{T}}{\rv{T} = T | n(T) = 0} \times \\
\left\{
\indicator{\left(\mdpvaluefunc{\mathcal{A}^\text{L}\left(T\right)}{\mathcal{M}_1}{\xi^{h_0}} \ge \mdpvaluefunc{\pi^*}{\mathcal{M}_1}{\xi^{h_0}} - \epsilon\right)} +
\indicator{\left(\mdpvaluefunc{\mathcal{A}^\text{L}\left(T\right)}{\mathcal{M}_2}{\xi^{h_0}} \ge \mdpvaluefunc{\pi^*}{\mathcal{M}_2}{\xi^{h_0}} - \epsilon\right)}
\right\}.
\end{multline*}
Plugging in \Cref{eq:m1m2valuessumtoonecorollary}, we get
\begin{equation*}
\zeta_1 + \zeta_2 \le 1 .
\end{equation*}
This then implies that $\min(\zeta_1, \zeta_2) \le \frac{1}{2}$.
\end{proof}

\subsection{Proof of \Cref{thm:m1m2valuessumtoone}}
\label{sec:proof-m1m2valuessumtoone}
\begin{proof}
We first focus on $\mathcal{M}_1$.
Consider an infinite run $\tau = (s_0, a_0, s_1, a_1, \dots)$ of the policy $\pi$ on $\mathcal{M}_1$.
Let $\wprefix{\tau}{i}$ denote the partial history up to state $s_i$;
let $w$ denote all the states $(s_0, s_1, \dots)$ in $\tau$.
Let $E_i$ denote the event that the visited state at step $i$ is either $h_0$ or $q_0$: $\windex{w}{i} \in \left\{h_0, q_0\right\}$.
We have:
\begin{align*}
\mdpvaluefunc{\pi}{\mathcal{M}_1}{\xi^{h_0}}
& = \prob{}{\wmap{\mathcal{L}}{w} \vDash \eventually h_0} \\
& = \sum_{i=1}^\infty \prob{}{\wmap{\mathcal{L}}{w} \vDash \eventually h_0 | E_i} \cdot \prob{}{E_i}
\end{align*}
Given that $E_i$ happens, the previous state $\windex{w}{i-1}$ must be $g_l$.
Then, the probability of satisfying the formula given the event $E_i$ is the probability of the learned policy choosing $a_1$ from the state $g_l$ after observing the partial history $\windex{\tau}{i-1}$:
\begin{equation}
\label{eq:M1valueinfsum}
\mdpvaluefunc{\pi}{\mathcal{M}_1}{\xi^{h_0}} 
= \sum_{i=1}^\infty \prob{}{\pi\left(\wprefix{\tau}{i-1}\right) = a_1 | E_i} \cdot \prob{}{E_i}.
\end{equation}
Symmetrically for $\mathcal{M}_2$ we then have:
\begin{equation}
\label{eq:M2valueinfsum}
\mdpvaluefunc{\pi}{\mathcal{M}_2}{\xi^{h_0}} 
= \sum_{i=1}^\infty \prob{}{\pi\left(\wprefix{\tau}{i-1}\right) = a_2 | E_i} \cdot \prob{}{E_i}.
\end{equation}

For any policy and any given partial history, the probability of choosing $a_1$ or $a_2$ must sum to $1$, that is:
\begin{equation*}
\small
\prob{}{\pi\left(\wprefix{\tau}{i-1}\right) = a_1 | E_i} + \prob{}{\pi\left(\wprefix{\tau}{i-1}\right) = a_2 | E_i} = 1
\end{equation*}
Therefore, we may add \Cref{eq:M1valueinfsum} and \Cref{eq:M2valueinfsum} to get:
\begin{equation*}
\mdpvaluefunc{\pi}{\mathcal{M}_1}{\xi^{h_0}} +
\mdpvaluefunc{\pi}{\mathcal{M}_2}{\xi^{h_0}} 
= \sum_{i=1}^\infty 1 \cdot \prob{}{E_i}.
\end{equation*}

Finally, since the event $E_i$ must happen for some finitary $i$ with probability $1$ (i.e., either $h_0$ or $q_0$ must be reached eventually with probability $1$), 
the expression on the right of the equation sums to $1$.
\end{proof}

\section{Complete Proof of \Cref{thm:ltlpaclowerboundFh0}}
\label{sec:ltlpaclowerboundFh0fullproof}

\begin{proof}
First, consider $\mathcal{M}_1$. 
We will derive a lower bound for $N$.
We begin by asserting that the inequality of \Cref{eq:pac-ltl-epsilon-delta-inequality} holds true for a reinforcement-learning algorithm $\mathcal{A} = (\mathcal{A}^\text{S}, \mathcal{A}^\text{L})$. That is:
\begin{equation*}
\prob*{\rv{T}}{
\mdpvaluefunc{\mathcal{A}^\text{L}(T)}{\mathcal{M}_1}{\xi^{h_0}} \ge \mdpvaluefunc{\pi^*}{\mathcal{M}_1}{\xi^{h_0}} - \epsilon}
\ge 1 - \delta.
\end{equation*}
We expand the left-hand side by conditioning on $n(T) = 0$:
\begin{small}
\begin{multline*}
\prob*{\rv{T}}{
\mdpvaluefunc{\mathcal{A}^\text{L}(T)}{\mathcal{M}_1}{\xi^{h_0}} \ge \mdpvaluefunc{\pi^*}{\mathcal{M}_1}{\xi^{h_0}} - \epsilon | n\left(T\right) = 0} \prob*{\rv{T}}{n\left(T\right) = 0} + \\
\prob*{\rv{T}}{
\mdpvaluefunc{\mathcal{A}^\text{L}(T)}{\mathcal{M}_1}{\xi^{h_0}} \ge \mdpvaluefunc{\pi^*}{\mathcal{M}_1}{\xi^{h_0}} - \epsilon | n\left(T\right) > 0} \left(1 - \prob*{\rv{T}}{n\left(T\right) = 0}\right) \\
\ge 1 - \delta.
\end{multline*}
\end{small}
Since $\prob*{\rv{T}}{
\mdpvaluefunc{\mathcal{A}^\text{L}(T)}{\mathcal{M}_1}{\xi^{h_0}} \ge 1 - \epsilon | n\left(T\right) > 0} \le  1$, we may relax the inequality to:
\begin{equation*}
(1 - \zeta_1) \prob*{\rv{T}}{n\left(T\right) = 0} \le \delta ,
\end{equation*}
where we also plugged in our definition of $\zeta_i$ (see \Cref{thm:minofm1m2pacprob}).
This relaxation optimistically assumes that a reinforcement-learning algorithm can learn an $\epsilon$-optimal policy by observing at least one transition to $h_0$ or $q_0$.

Since there are at most $N$ transitions initiating from the state $g_l$, and $n(T) = 0$ only occurs when all those transitions end up in $g_k$, we have $\prob*{\rv{T}}{n\left(T\right) = 0} \ge (1-p)^N$.
Incorporating this into the inequality we have:
\begin{equation*}
\left(1 - \zeta_1 \right) (1-p)^N
\le \delta .
\end{equation*}
Symmetrically, for $\mathcal{M}_2$ we have:
\begin{equation*}
\left(1 - \zeta_2 \right) (1-p)^N
\le \delta .
\end{equation*}
Since both inequalities need to hold, we may combine them by using $\min$ to choose the tighter inequality:
\begin{equation*}
\left(1 - \min\left(\zeta_1, \zeta_2\right) \right) (1-p)^N
\le \delta.
\end{equation*}
By applying \Cref{thm:minofm1m2pacprob}, we remove the inequality's dependence on $\zeta_i$, and get the desired lower bound of
\begin{equation*}
N \ge \frac{ \log (2\delta) }{\log \left(1-p\right)} ,
\end{equation*}
which completes the proof of \Cref{thm:ltlpaclowerboundFh0}.
\end{proof}

\section{Uncommittable Words for \texorpdfstring{non-\ltlclass{finitary}}{non-Finitary} Formulas}
\label{sec:proofofltlnonfinitaryprop}

In this section, we prove the following lemma:
\begin{lemma}
\label{thm:ltlnonguaranteenonsafetyprop}
Any LTL formula $\phi\notinltlclass{guarantee}$ has an uncommittable accepting word.
Any LTL formula $\phi\notinltlclass{safety}$ has an uncommittable rejecting word.
\end{lemma}

\subsection{Preliminaries}
We will review some preliminaries to prepare for our proof of \Cref{thm:ltlnonguaranteenonsafetyprop}.

We will use an automaton-based argument for our proof of \Cref{thm:ltlnonguaranteenonsafetyprop}.
To that end, we recall the following definitions for automatons.

\paragraph{Deterministic Finite Automaton}
A deterministic finite automaton (DFA) is a tuple $(S, A, P, s_0, s_\text{acc})$, where $(S, A, P, s_0)$ is a deterministic MDP (i.e., $P$ degenerated to a deterministic function $\functiontype{S,A}{S}$), and $s_\text{acc} \in S$ is an accepting state.

\paragraph{Deterministic Rabin Automaton} 
A deterministic Rabin automaton (DRA) is a tuple $(S, \Pi, T, s_0, \textit{Acc})$, where 
\begin{itemize}[leftmargin=*]
\item $S$ is a finite set of states.
\item $\Pi$ is the atomic propositions of $\phi$.
\item $T$ is a transition function 
$\functiontype{S, 2^\Pi}{S}$.
\item $s_0 \in S$ is an initial state.
\item \textit{Acc} is a set of pairs of subsets of states $(B_i, G_i) \in (2^S)^2$.
\end{itemize}

An infinite-length word $w$ over the atomic propositions $\Pi$ is accepted by the DRA, if there exists a run of the DRA such that there exists a $(B_i, G_i) \in \textit{Acc}$ where the run visits all states in $B_i$ finitely many times and visits some state(s) in $G_i$ infinitely many times.

For any LTL formula $\phi$, one can always construct an equivalent DRA that accepts the same set of infinite-length words as $\phi$ \citep{safraltl2rabinaut}.

\subsection{Proof of \Cref{thm:ltlnonguaranteenonsafetyprop} for \texorpdfstring{$\phi\notinltlclass{guarantee}$}{phi not in Guarantee}}

Given an LTL formula $\phi$, we first construct its equivalent DRA $\mathcal{R} = (S, \Pi, T, s_0, \textit{Acc})$ \citep{safraltl2rabinaut}. 

A {\em path} in a DRA is a sequence of transitions in the DRA.
A {\em cycle} in a DRA is a path that starts from some state and then returns to that state. 
A cycle is {\em accepting} if there exists a pair $(B_i, G_i) \in \textit{Acc}$, such that the cycle does not visit states in $B_i$ and visits some states in $G_i$.
Conversely, a cycle is {\em rejecting} if it is not accepting.
With the above definitions and to the end of proving \Cref{thm:ltlnonguaranteenonsafetyprop} for the case of $\phi\notinltlclass{guarantee}$, we state and prove the following lemma.
\begin{lemma}
\label{thm:guaranteemusthaveacceptingcycletorejectingcycle}
For any LTL formula $\phi\notinltlclass{guarantee}$ and its equivalent DRA $\mathcal{R}$, it must be the case that $\mathcal{R}$ contains an accepting cycle that is reachable from the initial state and there exists a path from a state in the accepting cycle to a rejecting cycle.
\end{lemma}
\begin{proof}
Suppose, for the sake of contradiction, there does not exist an accepting cycle that \begin{enumerate*}
\item is reachable from the initial state and
\item has a path to a rejecting cycle in the equivalent DRA $\mathcal{R}$.
\end{enumerate*}
Then there are two scenarios:
\begin{itemize}[leftmargin=*]
\item $\mathcal{R}$ does not have any accepting cycle that is reachable from the initial state.
\item All accepting cycles reachable from the initial state do not have any path to any rejecting cycle.
\end{itemize}

For the first scenario, $\mathcal{R}$ must not accept any infinite-length word. Therefore $\phi$ must be equal to $\ltlfalse$ (i.e., the constant falsum).
However, $\ltlfalse$ is in the \ltlclass{finitary} LTL class, which is a subset of \ltlclass{guarantee}, so this is a contradiction.

For the second scenario, consider any infinite-length word $w$.
Consider the induced infinite path $\mathcal{P} = (s_0, \windex{w}{0}, s_1, \windex{w}{1}, \dots)$ by $w$ on the DRA starting from the initial state $s_0$.

If $\phi$ accepts the word $w$, the path $\mathcal{P}$ must reach some state in some accepting cycle.

Conversely, if $\phi$ rejects the word $w$, the path must not visit any state in any accepting cycle. This is because otherwise the path can no longer visit a rejecting cycle once it visits the accepting cycle, thereby causing the word to be accepted.

Therefore, $\phi$ accepts the word $w$ as soon as the path $\mathcal{P}$ visits some state in some accepting cycle. 
This degenerates the DRA to a DFA, where the accepting states are all the states in the accepting cycles of the DRA.
Then, an infinite-length word $w$ is accepted by $\phi$ if and only if there exists a prefix of $w$ that is accepted by the DFA.

By the property of the \ltlclass{guarantee} class (see \Cref{sec:ltl-background-full}), for $\phi\inltlclass{guarantee}$, there exists a language of finite-length words, $L$, such that $w \vDash \phi$ if $L$ accepts a prefix of $w$ \citep{tlhierarchy}.
Since a DFA recognizes a regular language, the formula must be in the \ltlclass{guarantee} LTL class. This is also a contradiction.

Therefore, there must exist an accepting cycle that is reachable from the initial state and has a path to a rejecting cycle in the equivalent DRA.
This completes the proof of \Cref{thm:guaranteemusthaveacceptingcycletorejectingcycle}
\end{proof}

We are now ready to give a construction of $w_a$, $w_b$, $w_c$ and $w_d$ that directly proves \Cref{thm:ltlnonguaranteenonsafetyprop} for $\phi\notinltlclass{guarantee}$.
Consider the equivalent DRA $\mathcal{R}$ of the LTL formula. 
By \Cref{thm:guaranteemusthaveacceptingcycletorejectingcycle}, $\mathcal{R}$ must contain an accepting cycle that is reachable from the initial state and has a path to a rejecting cycle. We can thus define the following paths and cycles:
\begin{itemize}[leftmargin=*]
\item Let $\mathcal{P}_a$ be the path from the initial state to the accepting cycle.
\item Let $\mathcal{P}_b$ be the accepting cycle.
\item Let $\mathcal{P}_c$ be a path from the last state in the accepting cycle to the rejecting cycle. 
\item Let $\mathcal{P}_d$ be the rejecting cycle.
\end{itemize}
For a path $\mathcal{P} = (s_i, \windex{w}{i}, \dots s_j, \windex{w}{j}, s_{j+1})$, let  $w(\mathcal{P})$ denote the finite-length word consisting only of the characters in between every other state (i.e., each character is a tuple of truth values of the atomic propositions): $w(\mathcal{P}) = \windex{w}{i}\dots\windex{w}{j}$.
Consider the assignments of $w_a = w(\mathcal{P}_a)$,
$w_b = w(\mathcal{P}_b)$,
$w_c = w(\mathcal{P}_c)$ and
$w_d = w(\mathcal{P}_d)$.
Notice that:
\begin{itemize}[leftmargin=*]
\item The formula $\phi$ accepts the infinite-length word $\wrational{w_a}{w_b}$ because $P^b$ is an accepting cycle.
\item The formula $\phi$ rejects all infinite-length words 
$\wrational{w_a; \wrepeat{w_b}{i}; w_c}{w_d}$ for all $i\in\naturals$ because $P^d$ is a rejecting cycle.
\end{itemize}
By \Cref{def:uncommittable-word}, the infinite-length word $\wrational{w_a}{w_b}$ is an uncommittable accepting word. 
This construction proves  \Cref{thm:ltlnonguaranteenonsafetyprop} for $\phi\notinltlclass{guarantee}$.

\subsection{Proof of \Cref{thm:ltlnonguaranteenonsafetyprop} for \texorpdfstring{$\phi\notinltlclass{safety}$}{phi not in safety}}

The proof for $\phi\notinltlclass{safety}$ is symmetrical to $\phi\notinltlclass{guarantee}$.
For completeness, we give the proof below.

Given an LTL formula $\phi$, we again first construct its equivalent DRA $\mathcal{R} = (S, \Pi, T, s_0, \textit{Acc})$. 

To the end of proving \Cref{thm:ltlnonguaranteenonsafetyprop} for the case of $\phi\notinltlclass{safety}$, we state and prove the following lemma.
\begin{lemma}
\label{thm:safetymusthaverejectingcycletoacceptingcycle}
For any LTL formula $\phi \notinltlclass{safety}$ and its equivalent DRA $\mathcal{R}$, it must be the case that $\mathcal{R}$ contains a rejecting cycle that is reachable from the initial state and has a path from any state in the rejecting cycle to an accepting cycle.
\end{lemma}
\begin{proof}
Suppose, for the sake of contradiction, there does not exist a rejecting cycle that \begin{enumerate*}
\item is reachable from the initial state and
\item has a path to an accepting cycle in the equivalent DRA $\mathcal{R}$.
\end{enumerate*}
Then there are two scenarios:
\begin{itemize}[leftmargin=*]
\item $\mathcal{R}$ does not have any rejecting cycle that is reachable from the initial state.
\item All rejecting cycles reachable from the initial state do not have any path to any accepting cycle.
\end{itemize}

For the first scenario, $\mathcal{R}$ must not reject any infinite-length word. Therefore $\phi$ must be equal to $\ltltrue$ (i.e., the constant truth).
However, $\ltltrue$ is in the \ltlclass{finitary} LTL class, which is a subset of \ltlclass{safety}, so this is a contradiction.

For the second scenario, consider any infinite-length word $w$.
Consider the induced infinite path $\mathcal{P} = (s_0, \windex{w}{0}, s_1, \windex{w}{1}, \dots)$ by $w$ on the DRA starting from the initial state $s_0$.

If $\phi$ rejects the word $w$, the path $\mathcal{P}$ must reach some state in some rejecting cycle.

Conversely, if $\phi$ accepts $w$, the path must not visit any state in any rejecting cycle. This is because otherwise the path can no longer visit a accepting cycle once it visits the rejecting cycle, thereby causing the word to be rejected.

Therefore, $\phi$ rejects the word $w$ as soon as the path $\mathcal{P}$ visits some state in some rejecting cycle. 
We can again construct a DFA based on the DRA by letting the accepting states be all the states except those in a rejecting cycle of the DRA.
By this construction, $\phi$ accepts an infinite-length word $w$ if and only if the DFA accepts all finite-length prefixes of $w$.

By the property of the \ltlclass{safety} class (see \Cref{sec:ltl-background-full}), for $\phi\inltlclass{safety}$, there exists a language of finite-length words, $L$, such that $w \vDash \phi$ if $L$ accepts all prefixes of $w$ \citep{tlhierarchy}.
Since a DFA recognizes a regular language, the formula must be in the \ltlclass{safety} LTL class. This is also a contradiction.

Therefore, there must exist a rejecting cycle that is reachable from the initial state and has a path to an accepting cycle in the equivalent DRA.
This completes the proof of \Cref{thm:safetymusthaverejectingcycletoacceptingcycle}
\end{proof}

We are now ready to give a construction of $w_a$, $w_b$, $w_c$ and $w_d$ that directly proves \Cref{thm:ltlnonguaranteenonsafetyprop} for $\phi\notinltlclass{safety}$.
Consider the equivalent DRA $\mathcal{R}$ of the LTL formula. 
By \Cref{thm:safetymusthaverejectingcycletoacceptingcycle}, $\mathcal{R}$ must contain a rejecting cycle that is reachable from the initial state and has a path to an accepting cycle. We can thus define the following paths and cycles:
\begin{itemize}[leftmargin=*]
\item Let $\mathcal{P}_a$ be the path from the initial state to the rejecting cycle.
\item Let $\mathcal{P}_b$ be the rejecting cycle.
\item Let $\mathcal{P}_c$ be a path from the last state in the rejecting cycle to the accepting cycle. 
\item Let $\mathcal{P}_d$ be the accepting cycle.
\end{itemize}
Consider the assignments of $w_a = w(\mathcal{P}_a)$,
$w_b = w(\mathcal{P}_b)$,
$w_c = w(\mathcal{P}_c)$ and
$w_d = w(\mathcal{P}_d)$.
Notice that:
\begin{itemize}[leftmargin=*]
\item The formula $\phi$ rejects the infinite-length word $\wrational{w_a}{w_b}$ because $P^b$ is a rejecting cycle.
\item The formula $\phi$ accepts all infinite-length words 
$\wrational{w_a; \wrepeat{w_b}{i}; w_c}{w_d}$ for all $i\in\naturals$ because $P^d$ is an accepting cycle.
\end{itemize}
By \Cref{def:uncommittable-word}, the infinite-length word $\wrational{w_a}{w_b}$ is an uncommittable rejecting word. 
This construction proves  \Cref{thm:ltlnonguaranteenonsafetyprop} for $\phi\notinltlclass{safety}$.

\section{Proof of \Cref{thm:ltl-non-pac}: the Reverse Direction}
\label{sec:ltlnonpacreverseproof}
In this section, we give a proof to the reverse direction of \Cref{thm:ltl-non-pac}. 

\subsection{Proof}
\paragraph{Reduction to Infinite-horizon Cumulative Rewards}
Given an LTL formula $\phi$ in \ltlclass{finitary} with atomic propositions $\Pi$, one can compile $\phi$ into a DFA $\bar{\mathcal{M}}= (\bar{S}, 2^\Pi, \bar{P}, \bar{s_0}, \bar{s_\text{acc}})$ that decides the satisfaction of $\phi$ \citep{emcsafety}. 
In particular, for a given sample path $w$ of DTMC induced by a policy and the environment MDP, $\mathcal{L}(w)$ satisfies $\phi$ if and only if the DFA, upon consuming  $\mathcal{L}(w)$, eventually reaches the accept state $s_\text{acc}$.
Here, the DFA has a size (in the worst case) doubly exponential to the size of the formula: $|\bar{S}| = \BigO(\doubleexp(\abs{\phi}))$ \citep{mcsafety}.

We then use the following product construction to form an augmented MDP with rewards $\hat{\mathcal{M}} = (\hat{S}, \hat{A}, \hat{P}, \hat{s_0}, \hat{R})$.
Specifically, 
\begin{itemize}[leftmargin=*]
\item The states and actions are: $\hat{S} = S \times \bar{S}$ and $\hat{A} = A$.
\item The transitions follow the transitions in the environment MDP and the DFA simultaneously, where the action input of the DFA come from labeling the current state of the environment MDP. In particular, the transitions in the augmented MDP follows the equations: $\hat{P}((s, \bar{s}), a, (s', \bar{s}')) = P(s, a, s')$ and $\bar{s}' = \bar{P}(\bar{s}, \mathcal{L}(s))$.
\item The reward function assigns a reward of one to any transition from a non-accepting state that reaches $s_\text{acc}$ in the DFA, and zero otherwise:
$R((s, \hat{s}), a, (s', \hat{s}')) = \indicator{s \neq s_\text{acc} \land \hat{s}' = s_\text{acc}}$.
\end{itemize}

By construction, each run of the augmented MDP gives a reward of $1$ iff the run satisfies the finitary formula $\phi$.
The expected (undiscounted) infinite-horizon cumulative rewards thus equals the satisfaction probability of the formula.
Therefore, maximizing the infinite-horizon cumulative rewards in the augmented MDP is equivalent to maximizing the satisfaction probability of $\phi$ in the environment MDP.

\paragraph{Reduction to Finite-horizon Cumulative Rewards}
By the property of LTL hierarchy \citep{tlhierarchy}, for any LTL formula $\phi$ in \ltlclass{finitary} and an infinite-length word $w$, one can decide if $\phi$ accepts $w$ by inspecting a length-$H$ prefix of $w$. 
Here, $H$ is a constant that is computable from $\phi$. 
In particular, $H$ equals the longest distance from the start state to a terminal state in our constructed DFA. \footnote{Note that since $\phi$ is \ltlclasst{finitary}, the DFA does not have any cycles except at the terminal states \citep{spot}.}
Thus, since the product construction above does not assign any reward after the horizon $H$, the infinite-horizon cumulative rewards is further equivalent to the finite-horizon (of length $H$) cumulative rewards.
Therefore, finding the optimal policy for $\phi$ is equivalent to finding the optimal policy that maximizes the cumulative rewards for a finite horizon $H$ in the augmented MDP.

\paragraph{Sample Complexity Upper Bound}
Lastly, \citet{ipoc19} gave a \reinforcementlearning algorithm for finite-horizon cumulative rewards called ORLC (optimistic reinforcement learning with certificates).
The ORLC algorithm is sample efficiently PAC \footnote{The ORLC algorithm provides a guarantee called individual policy certificates (IPOC) bound. \citet{ipoc19} showed that this guarantee implies our PAC definition, which they called a supervised-learning style PAC bound.
Therefore, the ORLC algorithm is a PAC \reinforcementlearning algorithm for finite-horizon cumulative rewards by our definition.} and has a sample complexity of  $\tilde{\BigO}\left(\left(\frac{\abs{S}\abs{A}H^3}{\epsilon^2} +\frac{\abs{S}^2\abs{A}H^4}{\epsilon}  \right) \log \frac{1}{\delta}\right)$. \footnote{The notation $\tilde{\BigO}(.)$ is the same as $\BigO(.)$, but ignores any $\log$-terms.
The bound given by \citet{ipoc19} is $\tilde{\BigO}\left(\frac{\abs{S}^2\abs{A|}H^2}{\epsilon^2} \log \frac{1}{\delta}\right)$. 
 It is a upper bound on the number of episodes. 
 To make the bound consistent with our lower bound, which is a bound on the number of sampled transitions, we multiply it by an additional $H$ term.}
Incorporating the fact that 
the augmented MDP has $|\hat{S}| = \abs{S} \cdot \BigO(\doubleexp(\abs{\phi}))$ number of states, we obtain a sample complexity upper bound of $\tilde{\BigO}\left(\left(\frac{\abs{S}\doubleexp{\abs{\phi}}\abs{A}H^3}{\epsilon^2} +\frac{\abs{S}^2(\doubleexp{\abs{\phi}})^2\abs{A}H^4}{\epsilon}  \right) \log \frac{1}{\delta}\right)$ for the overall \reinforcementlearning algorithm.

Since for any \ltlclasst{finitary} formula, we have constructed a \reinforcementlearning algorithm that is sample efficiently LTL-PAC for all environment MDPs, this concludes our proof that any \ltlclasst{finitary} formula is LTL-PAC-learnable.

\section{Empirical Justifications}
\label{sec:empirical-appendix}

\DeclareDocumentCommand{\rewardscheme}{m}{%
\IfEqCase{#1}{%
    {multi-discount}{Multi-discount}%
    {zeta-reach}{Zeta-reach}%
    {zeta-acc}{Zeta-acc}%
    {zeta-discount}{Zeta-discount}%
    {reward-on-acc}{Reward-on-acc}%
}[\PackageError{rewardscheme}{Undefined option to \rewardscheme: #1}{}]%
}

\DeclareDocumentCommand{\rlalgo}{m}{%
\IfEqCase{#1}{%
    {Q}{Q-learning}%
    {DQ}{Double Q-learning}%
    {SL}{SARSA$\left(\lambda\right)$}%
}[\PackageError{rlalgo}{Undefined option to \rlalgo: #1}{}]%
}

This section empirically demonstrates our main result, the forward direction of \Cref{thm:ltl-non-pac}.

Previous work has introduced various \reinforcementlearning algorithms for LTL objectives \citep{dorsa,omegaregularrl19,hasanbeig2019reinforcement,bozkurt2020control}.
We therefore ask the research question:
{\em Do the sample complexities for \reinforcementlearning algorithms for LTL objectives introduced by previous work depend on the transition probabilities of the environment?}

To answer the above question, we 
consider a set of \reinforcementlearning algorithms for LTL objectives and empirically measure the sample size for each algorithm to obtain a near-optimal policy with high probability.

\subsection{Methodology}

\paragraph{Reinforcement-learning algorithms}
We consider a set of recent \reinforcementlearning algorithms for LTL objectives \citep{omegaregularrl19,bozkurt2020control}, about which we give more details in \Cref{sec:more_empirical}.
These algorithms are all implemented in the Mungojerrie toolbox \citep{hahn2021mungojerrie}.

\paragraph{Objectives and Environment MDPs}

\begin{figure}
\centering
\newcommand{\counterexamplemdpkind}{m1}
\begin{tikzpicture}[node distance=1.5cm,on grid,auto]
\tikzset{%
    in place/.style={
      auto=false,
      fill=white,
      inner sep=2pt,
    },
}

\newcommand{\statenodesize}{1.2em}

\tikzset{%
    every state/.style={
        fill={rgb:black,1;white,10},
        initial text=, 
        inner sep=0, 
        minimum size=\statenodesize,
        font={\small},
    },
    tl/.style={font=\small} %
}

\tikzset{%
    accstate/.style={
        state,
        pattern={Lines[angle=35,distance={4.5pt/sqrt(2)}]},
        pattern color=gray
    }
}

\tikzset{%
    rejstate/.style={state, fill=white}
}

\ifdefstring{\counterexamplemdpkind}{m12}{
\tikzset{%
    a1/.style={
      swap,
      auto=left,
      to path={ let \p1=(\tikztostart),\p2=(\tikztotarget), \n1={atan2(\y2-\y1,\x2-\x1)},\n2={\n1+180} in ($(\tikztostart.{\n1})!0.5mm!90:(\tikztotarget.{\n2})$) -- ($(\tikztotarget.{\n2})!0.5mm!270:(\tikztostart.{\n1})$) \tikztonodes},
    },
    a2/.style={
      auto=right,
      to path={ let \p1=(\tikztostart),\p2=(\tikztotarget), \n1={atan2(\y2-\y1,\x2-\x1)},\n2={\n1+180} in ($(\tikztostart.{\n1})!0.5mm!270:(\tikztotarget.{\n2})$) -- ($(\tikztotarget.{\n2})!0.5mm!90:(\tikztostart.{\n1})$) \tikztonodes}
    },
}
}

\ifdefstring{\counterexamplemdpkind}{m1}{%
\tikzset{
    a2/.style={draw=none, fill opacity=0},
    a1/.style={swap, auto=left}
}
}

\ifdefstring{\counterexamplemdpkind}{m2}{%
\tikzset{
    a2/.style={auto=right},
    a1/.style={draw=none, fill opacity=0}
}
}

\providecommand{\arrowheadscale}{0.8}
\tikzset{%
    m1/.style={
        -{Stealth[scale=\arrowheadscale]}
    },
    m2/.style={
        -{Triangle[open, scale=\arrowheadscale]}
    },
}

\ifdefstring{\counterexamplemdpkind}{m12}%
{
\tikzset{
    m12/.style={
        -{Stealth[scale=\arrowheadscale] Triangle[open, scale=\arrowheadscale]}
    }
}
}

\ifdefstring{\counterexamplemdpkind}{m1}%
{\tikzset{
    m12/.style={m1}
}}

\ifdefstring{\counterexamplemdpkind}{m2}%
{\tikzset{
    m12/.style={m2}
}}

\newcommand{\largedots}{$\textbf{\ldots}$}

    \tikzstyle{every state}=[
        fill={rgb:black,1;white,10},
        initial text=, 
        inner sep=0, 
        minimum size=1.2em,
        font={\small},
    ]
    
    \tikzstyle{tl/.style}=[font=\small] %

    \tikzstyle{every loop}=[
    style={
        looseness=1, 
        min distance=3mm,
        font={\small}
        }
    ]
    
    \node[state, initial left] (g) {$g$};
    \node[accstate] (h) [above right =1em and 6em of g]  {$h$};
    \node[rejstate] (q) [below right =1em and 6em of g] {$q$};

    \ifdefstring{\counterexamplemdpkind}{m1}{%
    \path[->]
    (g) edge [pos=0.9] node [tl] {$a_1,p$} (h)
    (g) edge [swap, pos=0.9] node [tl] {$a_2,p$} (q)
    ;
    }
    
    \ifdefstring{\counterexamplemdpkind}{m2}{%
    \path[->]
    (g) edge [pos=0.9] node [tl] {$a_2,p$} (h)
    (g) edge [swap, pos=0.9] node [tl] {$a_1,p$} (q)
    ;
    }
    
    \path[->] 
    (g) edge  [loop above]  node [tl] {$a_1,1 - p$}  ()
    (g) edge  [loop below]  node [tl] {$a_2,1 - p$}  ()
    
    (h) edge  [loop right]  node {}  ()
    (q) edge  [loop right]  node {}  ()
    
    ;

\end{tikzpicture}
\caption{One of the two environment MDPs used in the experiments.}
\label{fig:empirical-counterexample-mdp}
\end{figure}
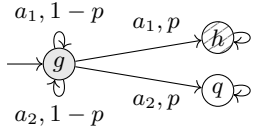

\begin{figure}
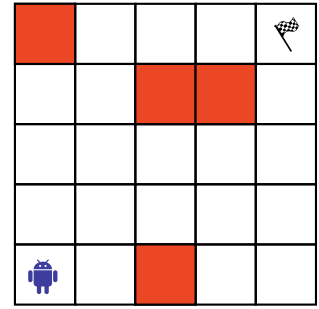

\centering
\begin{tikzpicture}[scale=0.8, transform shape]
\let\inputhidefromlatexpp\input

\tikzset{policy/.style={color=gray}}
\definecolor{slipcolor}{HTML}{a3cef1}
\definecolor{trapcolor}{HTML}{eb4724}
\tikzset{
lava/.pic={code={\input{lava.tikz}}},
treasure/.pic={code={\input{treasure.tikz}}},
goal/.pic={code={\input{goal.tikz}}},
bot/.pic={code={\input{bot.tikz}}}
}

\tikzset{
  box/.style={rectangle,draw=black,thick,minimum size=1cm},
  slip/.style={box,fill=slipcolor},
  trap/.style={box,fill=trapcolor},
  wall/.style={box,fill=gray},
}

\foreach \x in {0,...,4}{
    \foreach \y in {0,...,4}
        \node[box] at (\x,\y){};
}

\node[trap] at (0, 4) {};
\node[trap] at (2, 0) {};
\node[trap] at (2, 3) {};
\node[trap] at (3, 3) {};

\pic [local bounding box=P] at (0, 0) {bot};
\pic [local bounding box=G] at (4, 4) {goal};

\end{tikzpicture}
\caption{Gridworld environment MDP from \protect\citet{dorsa} with a customized transition dynamics.
The agent starts from the lower left corner.
At each time step, the agent can choose to move up, down, left or right.
The white cells are sticky: the agent moves towards the intended direction with probability $1 - p$ (or stays stationary if it will move off the grid), and stays stationary with probability $p$.
The red cells are trapping: once the agent steps on a red cell, it stays there forever.
}
\label{fig:empirical-counterexample-mdp-sadigh14}
\end{figure}

We consider two pairs of LTL formulas and environment MDPs (LTL-MDP pair).
The first pair is the formula $\eventually h$ and the counterexample MDP constructed according to  \Cref{sec:counterexample-mdps}, shown in \Cref{fig:empirical-counterexample-mdp}.
The second pair is the formula $\eventually \textit{goal}$ and a gridworld environment MDP from the case study by \citet{dorsa} with a customized transition dynamics, shown in \Cref{fig:empirical-counterexample-mdp-sadigh14}.

\paragraph{Experiment Methodology}
We ran the considered algorithms on each chosen LTL-MDP pair with a range of values for the parameter $p$ and let the algorithms perform $N$ environment samples.

For each algorithm and each pair of values of $p$ and $N$, we fix $\epsilon = 0.1$ and repeatedly run the algorithm to obtain a Monte Carlo estimation of the LTL-PAC probability (left side of \Cref{eq:pac-ltl-epsilon-delta-inequality}) for that setting of $p$, $N$ and $\epsilon$.
We repeat each setting until the estimated standard deviation of the estimated probability is within $0.01$. 
In the end, for each algorithm and LTL-MDP pair we obtain $5 \times 21 = 105$ LTL-PAC probabilities and their estimated standard deviations.

For the first LTL-MDP pair, we vary $p$ by a geometric progression from $10^{-1}$ to $10^{-3}$ in $5$ steps:
\begingroup\small$p(i) = 10^{-\frac{i+1}{2}}$\endgroup for \begingroup\small$1 \le i \le 5$\endgroup.
We vary $N$ by a geometric progression from $10^1$ to $10^5$ in $21$ steps.

For the second LTL-MDP pair, we vary $p$ by a geometric progression from $0.9$ to $0.6$ in $5$ steps:
\begingroup\small$p(i) = 0.9\times 0.903^{-i}$ \endgroup for \begingroup\small$1 \le i \le 5$\endgroup.
We vary $N$ by a geometric progression from $3540$ to $9\times10^4$ in $21$ steps; if an algorithm does not converge to the desired LTL-PAC probability within $9\times10^4$ steps, we rerun the experiment with an extended range of $N$ from $3540$ to $1.5\times10^5$. 

\subsection{Results}

\Cref{fig:bozkurt_pac_probs_plot} presents the results for the algorithm in \citet{bozkurt2020control} with the setting of \rewardscheme{multi-discount}, \rlalgo{Q}, and the first LTL-MDP pair.

On the left, we plot the LTL-PAC probabilities vs.\ the number of samples $N$, one curve for each $p$.
On the right, we plot the intersections of the curves in the left plot with a horizontal cutoff of $0.9$.

As we see from the left plot of \Cref{fig:bozkurt_pac_probs_plot}, for each $p$, the curve starts at $0$ and grows to $1$ in a sigmoidal shape as the number of samples increases.
However, as $p$ decreases, the MDP becomes harder: As shown on the right plot of \Cref{fig:bozkurt_pac_probs_plot}, the number of samples required to reach the particular LTL-PAC probability of $0.9$ grows exponentially.

\Cref{fig:complete-empirical-results} presents the complete results for all settings for the first LTL-MDP pair, and \Cref{fig:complete-empirical-results-sadigh14} present the complete results for all settings for the second LTL-MDP pair.
These results are similar and lead to the same analysis as above.

\newcommand{\includeempiricalplot}[3]{%
\begin{subfigure}[t]{0.48\textwidth}%
\centering%
\includegraphics[width=0.6\linewidth]{plots/#1-#2-#3.png}%
\includegraphics[width=0.38\linewidth]{plots/#1-#2-#3-intercept-0_1.png}%
\caption{\rewardscheme{#1} with \rlalgo{#2}}%
\end{subfigure}%
}

\begingroup
\begin{figure*}[t]
\centering
\foreach \ltlconstruction in {reward-on-acc,multi-discount,zeta-reach}{%
    \includeempiricalplot{\ltlconstruction}{Q}{rl_ltl_pac_paper}%
    \includeempiricalplot{\ltlconstruction}{DQ}{rl_ltl_pac_paper}\\%
    \includeempiricalplot{\ltlconstruction}{SL}{rl_ltl_pac_paper}\\%
}
\caption{Empirical results of the first LTL-MDP pair (continued on next page)}
\end{figure*}
\begin{figure*}[t]\ContinuedFloat
\centering
\foreach \ltlconstruction in {zeta-acc,zeta-discount}{%
    \includeempiricalplot{\ltlconstruction}{Q}{rl_ltl_pac_paper}%
    \includeempiricalplot{\ltlconstruction}{DQ}{rl_ltl_pac_paper}\\%
    \includeempiricalplot{\ltlconstruction}{SL}{rl_ltl_pac_paper}\\%
}
\caption{Empirical results of the first LTL-MDP pair (continued). Each sub-figure corresponds to a specific reward-scheme and learning-algorithm pair. For each sub-figure, on the left: LTL-PAC probabilities vs. number of samples, for varying parameters $p$; on the right: number of samples needed to reach 0.9 LTL-PAC probability vs.  parameters $p$.
}
\label{fig:complete-empirical-results}
\end{figure*}
\endgroup

\begingroup
\begin{figure*}[t]
\centering
\foreach \ltlconstruction in {reward-on-acc,multi-discount,zeta-reach}{%
    \includeempiricalplot{\ltlconstruction}{Q}{gridworld_sadigh14}%
    \includeempiricalplot{\ltlconstruction}{DQ}{gridworld_sadigh14}\\%
    \includeempiricalplot{\ltlconstruction}{SL}{gridworld_sadigh14}\\%
}
\caption{Empirical results of the second LTL-MDP pair (continued on next page)}
\end{figure*}
\begin{figure*}[t]\ContinuedFloat
\centering
\foreach \ltlconstruction in {zeta-acc,zeta-discount}{%
    \includeempiricalplot{\ltlconstruction}{Q}{gridworld_sadigh14}%
    \includeempiricalplot{\ltlconstruction}{DQ}{gridworld_sadigh14}\\%
    \includeempiricalplot{\ltlconstruction}{SL}{gridworld_sadigh14}\\%
}
\caption{Empirical results of the second LTL-MDP pair (continued). Each sub-figure corresponds to a specific reward-scheme and learning-algorithm pair. For each sub-figure, on the left: LTL-PAC probabilities vs. number of samples, for varying parameters $p$; on the right: number of samples needed to reach 0.9 LTL-PAC probability vs.  parameters $p$.
}
\label{fig:complete-empirical-results-sadigh14}
\end{figure*}
\endgroup

\subsection{Conclusion} 
Since the transition probabilities ($p$ in this case) are unknown in practice, one can't know which curve in the left plot a given environment will follow.
Therefore, given any finite number of samples, these reinforcement-algorithms cannot provide guarantees on the LTL-PAC probability of the learned policy.
This result supports \Cref{thm:ltl-non-pac}.

\section{Empirical Experiment Details}
\label{sec:more_empirical}

\begin{table*}[t]
\centering
\ra{1.3}
\begin{tabular}{@{}llll@{}}
\toprule
Reinforcement-learning-algorithm & Learning Rate & Exploration &  Reset Episode Every Steps \\
\midrule
\rlalgo{Q}  & $\frac{10}{10+t}$ & Linear decay from $1.0$ to $10^{-1}$ & 10\\
\rlalgo{DQ} & $\frac{30}{30+t}$ & Linear decay from $1.0$ to $10^{-1}$  & 10 \\
\rlalgo{SL} & $\frac{10}{10+t}$ & Linear decay from $1.0$ to $10^{-3}$ & 10\\
\bottomrule
\end{tabular}
\caption{Non-default hyper-parameters used for each learning-algorithm}
\label{tab:reward-schemes-parameters}
\end{table*}

\subsection{Details of Methodology}

\paragraph{Chosen Algorithms}

We consider a set of recent reinforcement-learning algorithms for LTL objectives implemented in the Mongujerrie toolbox \citep{hahn2021mungojerrie}.

A common pattern in these previous works \citep{dorsa,omegaregularrl19,bozkurt2020control} is that each work constructs a product MDP with rewards (i.e., an MDP with a reward function on that MDP) from an LTL formula and an environment MDP.
Moreover, these works permit the use of any standard reinforcement-learning algorithm, such as Q-learning or SARSA($\lambda$), to solve the constructed product MDP with the specified reward function to obtain the product MDP's optimal policy.
Finally, these works cast the optimal policy back to a non-Markovian policy of the environment MDP, which becomes the algorithm's output policy.

Following \citet{hahn2021mungojerrie}, we call each specific construction of a product MDP with rewards as a {\em reward-scheme}.
We then characterize each reinforcement-learning algorithm as a ``reward-scheme'' and ``learning-algorithm'' pair.
We consider a total of five reward-schemes \footnote{We use the same naming of each reward-scheme as in the Mungojerrie toolbox  \citep{hahn2021mungojerrie}}: \rewardscheme{reward-on-acc} \citep{dorsa},  \rewardscheme{multi-discount} \citep{bozkurt2020control}, \rewardscheme{zeta-reach} \citep{omegaregularrl19}, \rewardscheme{zeta-acc}  \cite{omegaregularrl20}, and \rewardscheme{zeta-discount} \cite{omegaregularrl20}.
We consider a total of three learning-algorithms: \rlalgo{Q} \citep{qlearning92}, \rlalgo{DQ} \citep{Hasselt2010DoubleQ}, and \rlalgo{SL} \citep{sutton98sarsa}.
This yields a total of $15$ reinforcement-learning algorithms for LTL objectives.

\paragraph{Algorithm Parameters}

Each reinforcement-learning algorithm in Mungojerrie accepts a set of hyper-parameters.
For the majority of the hyper-parameters, we use their default values as in Mungojerrie Version 1.0 \citep{hahn2021mungojerrie}.
We present the hyper-parameters that differ from the default values in \Cref{tab:reward-schemes-parameters}.
For each of the hyper-parameters in \Cref{tab:reward-schemes-parameters}, we use a different value from the default value because it allow all the algorithms that we consider to converge within $10^{5}$ steps (i.e., the maximum learning steps that we allow). 
For SARSA$\left(\lambda\right)$, we use $\lambda = 0$.

\paragraph{Software and Platform}

We use a custom version of Mungojerrie. 
Our modifications are:
\begin{itemize}[leftmargin=*]
\item Modification to allow parallel Monte Carlo estimation of the LTL-PAC probability.
\item Modification to allow the reinforcement-learning algorithms to have a non-linear learning rate decay. 
In particular, we use a learning rate of $\frac{k}{k+t}$ at every learning step $t$, where $k$ is a hyper-parameter (see \Cref{tab:reward-schemes-parameters} for the value of $k$ for each algorithm). This modification is necessary for ensuring Q-learning's convergence \citep{qlearning92}.
\end{itemize}

We run all experiments on a machine with 2.9 GHz 6-Core CPU and 32 GB of RAM.

\section{Classification of Prior Works}

In \Cref{sec:directions_forward}, we discussed several categories of approaches that either focus on tractable objectives or weaken the guarantees required by an LTL-PAC algorithm.
We classified various previous approaches into these categories.
In this section, we explain the rationale for each classification.

\subsection{Use a \texorpdfstring{\ltlclasst{Finitary}}{Finitary} Objective}
\citet{smcbltl} introduced a variant of LTL called {\em Bounded LTL} and used Bounded LTL objective for reinforcement learning.
Every Bounded LTL formula is decidable by a bounded length prefix of the input word.
Moreover, each Bounded LTL formula is equivalent to an \ltlclasst{finitary} LTL formula.
Therefore, we classified this approach as using a \ltlclasst{Finitary} objective.

\citet{osbert} introduced a task specification language over finite-length words. 
Further, their definition of an MDP contains an additional finite time horizon $H$.
Each sample path of the MDP is then a length-$H$ finite-length word and is evaluated by a formula of the task specification language.\footnote{There are two possible interpretations of the finite horizon in \citet{osbert}.
The first interpretation is that the environment MDP inherently terminates and produces length-$H$ sample paths.
The second interpretation is that the finite horizon $H$ is part of the specification given by a user of their approach.
We used the second interpretation to classify their approach.
The difference between the two interpretations is only conceptual --- if the environment inherently terminates with a fixed finite horizon $H$, it would be equivalent to imposing a finite horizon $H$ in the task specification.
}
Each formula of the task specification language with a fixed finite horizon $H$ is equivalent to an LTL formula in the \ltlclasst{Finitary} class.
Therefore, we classified this approach as using a \ltlclasst{Finitary} objective.

\subsection{Best-Effort Guarantee}

We classified \citet{pacmdpsmcconfidenceinterval}
to this category and explained our rationale of this classification in \Cref{sec:directions_forward}.

\subsection{Know More About the Environment}

We classified \citet{ufukpacmdpltl,smcpmin} to this category and explained our rationale of this classification in \Cref{sec:directions_forward}.

\subsection{Use an LTL-like Objective}

\subsubsection{LTL-in-the-limit Objectives}

We classified \citet{dorsa,omegaregularrl19,hasanbeig2019reinforcement,bozkurt2020control} as using LTL-like objectives, and explained our rationale of these classifications in \Cref{sec:directions_forward}.

\subsubsection{General LTL-like Objectives}

\citet{gltl17} introduced a discounted variant of LTL called {\em Geometric LTL} (GLTL). A temporal operator in a GLTL formula expires within a time window whose length follows a geometric distribution. 
For example, a GLTL formula $\eventually_{0.1} \textit{goal}$ is satisfied if the sample path reaches the $\textit{goal}$ within a time horizon $H$, where $H$ follows $\text{Geometric}(0.1)$, the geometric distribution with the success parameter $0.1$.
Since GLTL's semantics is different from LTL's semantics, we classified this approach as using an LTL-like objective.

\citet{truncatedltl-rl} introduced a variant of LTL called {\em Truncated-LTL} (TLTL). 
A formula in TLTL, similar to a formula in Bounded LTL \citep{smcbltl}, is decidable by a bounded length prefix of the input word. 
Moreover, TLTL has a {\em qualitative semantics}, in addition to the standard Boolean semantics of LTL.
In particular, the qualitative semantics of a TLTL formula maps a sample path of the environment MDP to a real number that indicates the degree of satisfaction for the TLTL formula.
Therefore, we classified this approach as using an LTL-like objective.

\citet{ltlf-rl} introduced {\em Restraining Bolts}. A Restraining Bolts specification is a set of pairs $(\phi_i, r_i)$, where each $\phi_i$ is an LTLf/LDLf formula, and $r_i$ is a scalar reward. 
An LTLf formula is visuaully similar to an LTL formula; however, it is interpreted over finite-length words instead of infinite-length words.
LDLf is an extension of LTLf and is also interpreted over finite-length words. 
\footnote{LDLf is more expressive than LTLf \citep{ltlfldlf}. In particular, LTLf is equally expressive as the star-free subset of regular languages while LDLf is equally expressive as the full set of regular languages.}
Given an environment MDP, the approach checks each finite length prefix of a sample path of the MDP against each $\phi_i$, and if a prefix satisfies $\phi_i$, the approach gives the corresponding reward $r_i$ to the agent.
The objective in \citet{ltlf-rl} is to maximize the discounted cumulative rewards produced by the Restraining Bolts specification.
To the best of our knowledge, this objective is not equivalent to maximizing the satisfaction of an LTL formula.
Nonetheless, a Restraining Bolts specification can be seen as an LTL-like specification for its use of temporal operators.
Therefore, we classified this approach as using an LTL-like objective.

\citet{rewardmachine1} introduced {\em reward machine}.
A reward machine specification is a deterministic finite automaton equipped with a reward for each transition.
The objective in \citet{rewardmachine1} is to maximize the discounted cumulative rewards produced by the reward machine specification.
\citet{rewardmachine1} showed that LTL objectives formulas in the \ltlclasst{Guarantee} or \ltlclasst{Safety} class are reducible to reward machine objectives without discount factors.
However, since the approach maximizes discounted cumulative rewards in practice, it does not directly optimize for the LTL objectives in the \ltlclasst{Guarantee} or \ltlclasst{Safety} classes.\footnote{By their reduction, as the discount factor approach $1$ in the limit, the learned policy for the reward machine becomes the optimal policy for given \ltlclasst{guarantee} or \ltlclasst{safety} LTL objective. Therefore, \citet{rewardmachine1} can also be classified as using an LTL-in-the-limit objectives (for the subset of \ltlclasst{guarantee} and \ltlclasst{safety} objectives. Nonetheless, we classified this approach to the general LTL-like objectives category because reward machine objectives are more general than LTL objectives.}
Therefore, we classified this approach as using an LTL-like objective.

\section{Concurrent Work}
\label{sec:concurrent_work}

Concurrent to this work,
\citet{alur2021framework} developed a framework to study reductions between reinforcement-learning task specifications.
They looked at various task specifications, including cumulative discounted rewards, infinite-horizon average-rewards, reachability, safety, and LTL. 
They thoroughly review previous work concerning reinforcement learning for LTL objectives, which we also cite.
Moreover, \citet[Theorem 8]{alur2021framework} states a seemingly similar result as the forward direction of our \Cref{thm:ltl-non-pac}:
\begin{quote}
There does not exist a PAC-MDP algorithm for the class of safety specifications.
\end{quote}
Despite the parallels, we clarify one crucial difference and two nuances between our work and theirs.

Firstly and most importantly, their theorem is equivalent to
``there exists a safety specification that is not PAC-learnable.'', whereas our \Cref{thm:ltl-non-pac} works pointwise for each LTL formula, asserting ``all non-\ltlclasst{finitary} specifications are not PAC-learnable.''
The proof of their theorem gives one safety specification and shows that it is not PAC-learnable.\footnote{Their result is similar to what we showed in our \Cref{sec:samplecomplexityFh0}, where we consider the particular \ltlclasst{guarantee} formula $\eventually h_0$ and show that it is not PAC-learnable.}
On the other hand, the proof of the forward direction of our \Cref{thm:ltl-non-pac} constructs a counterexample for each non-\ltlclasst{finitary} formula.
This point is crucial since it allows us to precisely carve out the PAC-learnable subset, namely the \ltlclasst{finitary} formulas, from the LTL hierarchy. 
Secondly, their notion of sample complexity is slightly different from ours.
In particular, they formulated a reinforcement learning algorithm as an iterative algorithm. At each step, the iterative algorithm outputs a policy $\pi^n$.
Then, their notion of sample complexity is the total number of non-$\epsilon$-optimal policies produced during an infinitely long run of the learning algorithm:
$\left| \lbrace n \mid \mdpvaluefunc{\pi^n}{\mathcal{M}}{\xi} < \mdpvaluefunc{\pi^*}{\mathcal{M}}{\xi} - \epsilon \rbrace \right|.$ 
On the other hand, our notion of sample complexity is the number of samples required until the learning algorithm outputs $\epsilon$-optimal policies. 
However, this difference is orthogonal to the core issue caused by infinite-horizon LTL formulas.
In particular, we can adapt our theorem and proof to use their notion of sample complexity.

Thirdly, their definition of safety specification is equivalent to a strict subset of the \ltlclasst{safety} class in the LTL hierarchy that we consider.
In particular, their safety specification is equivalent to LTL formulas of the form $\always (a_1 \lor a_2 \lor \dots \lor a_n)$, where each $a_i \in \Pi$ is an atomic proposition, with $n = 0$ degenerating the specification to $\ltltrue$ (the constant true).

Lastly, they consider only \reinforcementlearning algorithms, whereas we consider the slightly more general \planningwithgenerativemodel algorithms.
We believe their theorem and proof can be modified to accommodate our more general algorithm definition.
\end{appendices}
\fi

\end{document}